\documentclass{article}

\usepackage{arxiv}

\usepackage[utf8]{inputenc} 
\usepackage[T1]{fontenc}    
\usepackage{hyperref}       
\usepackage{url}            
\usepackage{booktabs}       
\usepackage{amsfonts}       
\usepackage{nicefrac}       
\usepackage{microtype}      
\usepackage{lipsum}

\usepackage{graphicx}    
\usepackage{subcaption}  

\usepackage{amssymb}
\usepackage{amsmath}
\usepackage{amsthm}

\usepackage{algorithm}
\usepackage{algpseudocode}

\newtheorem{definition}{Definition}
\newtheorem{theorem}{Theorem}
\newtheorem{lemma}{Lemma}
\newtheorem{proposition}{Proposition}

\DeclareMathOperator*{\argmin}{arg\,min}

\usepackage{xparse}
\ExplSyntaxOn
\NewDocumentCommand{\Rowvec}{ O{,} m }
 {
  \vector_main:nnnn { p } { & } { #1 } { #2 }
 }
\NewDocumentCommand{\Colvec}{ O{,} m }
 {
  \vector_main:nnnn { p } { \\ } { #1 } { #2 }
 }

\seq_new:N \l__vector_arg_seq
\cs_new_protected:Npn \vector_main:nnnn #1 #2 #3 #4
 {
  \seq_set_split:Nnn \l__vector_arg_seq { #3 } { #4 }
  \begin{#1matrix}
  \seq_use:Nnnn \l__vector_arg_seq { #2 } { #2 } { #2 }
  \end{#1matrix}
 }
\ExplSyntaxOff

\title{R\'enyi Differentially Private ADMM for Non-Smooth Regularized Optimization}

\author{
  Chen Chen \\
  Department of Computer Science \\
  University of Georgia\\
  Athens, GA 30602\\
  \And
  Jaewoo Lee \\
  Department of Computer Science \\
  University of Georgia\\
  Athens, GA 30602\\
}

\begin{document}
\maketitle

\begin{abstract}

In this paper we consider the problem of minimizing composite
objective functions consisting of a convex differentiable loss
function plus a non-smooth regularization term, such as $L_1$ norm or
nuclear norm, under R\'enyi differential privacy (RDP). To solve the
problem, we propose two stochastic 
alternating direction method of multipliers (ADMM) algorithms:
ssADMM based on gradient perturbation and mpADMM based on output
perturbation. Both algorithms decompose the original problem into
sub-problems that have closed-form solutions. The first algorithm,
ssADMM, applies the recent privacy amplification result for RDP to
reduce the amount of noise to add. The second algorithm, mpADMM,
numerically computes the sensitivity of ADMM variable updates and
releases the updated parameter vector at the end of each epoch.
We compare the performance of our
algorithms with several baseline algorithms on both real and
simulated datasets.
Experimental results show that, in high privacy regimes (small
$\epsilon$), ssADMM and mpADMM outperform other baseline algorithms in 
terms of classification and feature selection performance,
respectively. 

\end{abstract}

\section{Introduction}
\label{sec:intro}

Concerns on privacy of individuals in the data used for training
machine learning models have led to extensive research on private
model building
techniques~\cite{chaudhuri2011differentially,kifer2012private,bassily2014private,abadi2016deep,zhang2017efficient,wang2017differentially,chen2019renyi},
especially in the context of Empirical Risk Minimization
(ERM).
Let $D = (d_1, d_2, \ldots, d_n)$ be a dataset, where $d_i \in \mathcal{D}$.
Many machine learning problems can be formulated as
regularized optimization problems of form:
\begin{equation} \label{eq:finite_sum_problem}
  \min_{x \in \mathbb R^p} F(x) := \frac{1}{n}\sum_{i=1}^n f(x, d_i) + \lambda h(x)
\end{equation}
where $\lambda > 0$ is a regularization coefficient, 
$f:\mathbb R^p\times \mathcal{D} \rightarrow \mathbb R$
is a smooth convex loss function, and $h:\mathbb R^p \rightarrow \mathbb R$ is a simple
convex \emph{non-smooth} regularizer such as $L_1$-norm or nuclear
norm. This formulation has received substantial attention as it arises 
in many interesting applications of machine learning such as generalized
lasso~\cite{tibshirani2011solution}, matrix
recovery~\cite{zhang2018primal,liu2016blessing}, and a class of $L_1$
regularized problems.
Despite recent advances in methods for differentially
private ERM, many existing solutions are not directly applicable to the
problem in~\eqref{eq:finite_sum_problem} due to requirement for
differentiability~\cite{bassily2014private,abadi2016deep,zhang2017efficient,chen2019renyi} or strong
convexity~\cite{chaudhuri2011differentially} of the regularization
term $h(x)$.
Alternating direction method of multipliers
(ADMM)~\cite{gabay1976dual} has shown to be effective in solving optimization
problems with complicated structure regularization.

In this paper, we propose two stochastic ADMM algorithms that satisfy
R\'enyi Differential Privacy (RDP), namely subsampled stochastic ADMM (ssADMM)
and model perturbation based ADMM (mpADMM). The first algorithm has
the following key features. First, ssADMM is scalable and fast. The
algorithm splits the composite objective function into 
differentiable and non-smooth terms, $\sum_{i} f(x, d_i)$ and $h(x)$,
using the ADMM framework. The differentiable term is further 
approximated by the first order Taylor expansion and linearization as
in~\cite{ouyang2013stochastic}. This approximated augmented
Lagrangian function has a simple analytical solution.
For the non-smooth regularization term $h(x)$, ssADMM applies proximal
mappings. For many non-smooth regularization function popularly used in
machine learning, such as $L_1$-norm, SCAD~\cite{fan2001variable}, and
MCP~\cite{zhang2010nearly}, those proximal mappings yield closed form
solutions. Therefore, both subproblems can be solved efficiently.

Second, ssADMM makes use of recently proposed
\emph{privacy amplification} lemma~\cite{wang2018subsampled} to tightly
bound the total privacy loss across many iterations.
In the closed-form solution of the modified augmented Lagrangian
function, the only data dependent term is the gradient $\nabla
f(x^k)$, where $x^k$ denotes the value of $x$ at iteration $k$. The
algorithm computes the gradient $\nabla f(x^k)$ using a randomly  
\emph{subsampled} data and add Gaussian noise to ensure $(\alpha,
\epsilon_k)$-RDP, which allows us to exploit the randomness in the
subsampling and to introduce less noise to each iteration.

The second algorithm, mpADMM, takes the output perturbation approach
but substantially differs from the original method.
Unlike the original method
which releases model parameters once only at the end, the proposed
method releases the output after each epoch. For each epoch, 
we numerically compute the sensitivity of both primal and dual
variable updates in ADMM and release the parameter vector using the
Gaussian mechanism. The algorithm uses the released (noisy) output as
the starting value for the next epoch.

Our contributions are summarized as follows: 
\begin{itemize}
    \item We propose two efficient R\'enyi differentially private
      algorithms, based on stochastic ADMM, for solving non-smooth
      convex optimization problems. In our proposed ssADMM, each
      subproblem is solved exactly in closed form.
    \item We apply the recent privacy amplification result for RDP to
      stochastic ADMM and show that the inherent randomness in
      subsampling process can be used to achieve stronger privacy
      protection.
    \item We empirically show the effectiveness of the proposed
      algorithms by performing extensive empirical evaluations on
      generalized linear models and comparing with other
      baseline algorithms.  The results show that, in high privacy
      regimes (small $\epsilon$), ssADMM and mpADMM outperform other
      baseline algorithms in terms of classification and feature
      selection performance, respectively.
\end{itemize}
The rest of this paper are organized as follow:
Section~\ref{sec:related} summarizes related work. In
Section~\ref{sec:prelim}, we provide background on R\'enyi
differential privacy and ADMM. Section~\ref{sec:algorithm} introduces
the proposed R\'enyi differentially private ADMM
algorithms. Section~\ref{sec:experiment} provides the performance
evaluations on both synthetic and real
datasets. Section~\ref{sec:conclusion} concludes the paper.  

\section{Related Work}
\label{sec:related}

Many works have been done to solve the empirical risk minimization
problem under differential privacy. Generally, there are three types
of algorithms proposed. Output perturbation algorithms perturb the model
parameters based on sensitivity, for example,
\cite{chaudhuri2011differentially} analyzed the sensitivity of optimal
solutions trained between neighboring databases; \cite{zhang2017efficient}
tackled the case when full gradient descent is applied; and
\cite{wu2017bolt} and \cite{chen2019renyi} analyzed the situation of
applying stochastic gradient descent on permuting
mini-batches. Objective perturbation algorithms perturb the training
objective functions, and the privacy guarantee is subject to an exact
solution of the ERM problem: \cite{chaudhuri2011differentially}
presented the first objective perturbation technique, and it is
extended by \cite{kifer2012private}. Gradient perturbation algorithms
perturb the (stochastic) gradients used for model updating by first-order
optimization methods, and use a composition technique to quantify the
overall privacy leak for multiple access of the data through gradient
calculation. For example, \cite{bassily2014private} proposed ``strong
composition'' theorem, then \cite{abadi2016deep} proposed ``moment
accountant'' method, which is also used in
\cite{wang2017differentially} and \cite{koskela2018learning}. The
R\'eyni differential privacy was introduced by
\cite{mironov2017renyi}, which can also be applied in gradient
perturbation, especially after \cite{wang2018subsampled} proposed its
amplification by subsampling results. 

Alternating Direction Method of Multipliers (ADMM) is an old algorithm to solve optimization problems \cite{boyd2011distributed}. It has been extensively studied, and applied in many domains such as outlier recovery \cite{tan2013traffic}, image processing \cite{chan2016plug}, and sensor detection \cite{dhingra2014admm}. In addition to its original version, many variations has been presented, such as \cite{esser2009applications, yang2013linearized} and \cite{ouyang2013stochastic}. Several ADMM based differentially private algorithms have been presented, for example, \cite{wang2019differential} applied objective perturbation technique on the original ADMM problem, \cite{zhang2017dynamic} and \cite{zhang2018improving} applied output and objective perturbation technique, and \cite{huang2019dp} applied gradient perturbation technique on ADMM-based algorithms in distributed settings.

$L_1$ regularized ERM problem was first proposed for linear regression, that is least absolute shrinkage and selection operator (LASSO) \cite{tibshirani1996regression}. Some variants of LASSO exists, such as \cite{zou2006adaptive} and \cite{park2008bayesian}. It has been used for classification problems, and many algorithms for solving $L_1$ regularized generalized linear models were presented, such as \cite{lee2006efficient}, \cite{park2007l1}, and \cite{bian2019parallel}. \cite{ng2004feature} and \cite{goodman2004exponential} has shown that $L_1$ regularized classification has good performance in feature selection. Limited to the assumption on the loss function, many differentially private ERM algorithms cannot be directly applied on $L_1$ regularized classification, with a few exceptions such as \cite{abadi2016deep, wang2019differential}, and \cite{huang2019dp}.

\section{Preliminaries}
\label{sec:prelim}

In this section we introduce relative background of this paper. We will start with definitions and lemmas in differential privacy and R\'enyi differential privacy, the $L_1$-regularized classification problem we aim to solve, and then the ADMM algorithm based on which we proposed our algorithms. 

We assume a dataset $D=\{d_1, ..., d_n\} \sim \mathcal D^n$ is a set collected from $n$ individuals from an unknown population distribution $\mathcal D$, where $d_i=(s_i, l_i)$ for $i=1,...,n$ is a record of one individual, with $s_i$ being a vector of features of dimension $p$, and $l_i\in\{-1, +1\}$ being its label. Two datasets $D$ and $D'$ are considered \emph{neighboring}, if $D'$ can be obtained by replacing one record with another one from $\mathcal D$, notated as $D\sim D'$. We use $x,y,z$ to denote model parameters, and $\|\cdot\|_1$ (resp. $\|\cdot\|_2$) as $L_1$ (resp. $L_2$) norm of a vector.

\subsection{Differential Privacy}

Differential privacy is so far the standard standard for protecting the privacy of sensitive datasets. Its formal definition is stated as:

\begin{definition}[$(\epsilon, \delta)$-Differential Privacy (DP)]
\cite{dwork2006calibrating} \cite{dwork2006our} Given privacy parameters $\epsilon\geq0, 0\leq\delta\leq 1$, a randomized mechanism (algorithm) $\mathcal M$ satisfies $(\epsilon, \delta)$-DP if for every event $S\subseteq range(M)$, and for every pair of neighboring datasets $D\sim D'$,
\begin{equation}
    \Pr[\mathcal M(D)\in S] \leq e^{\epsilon} \Pr[\mathcal M(D')\in S] + \delta
\end{equation}
\end{definition}

If $\delta=0$, it is called \emph{pure} differential privacy, and $\delta>0$ is called \emph{approximate} differential privacy.

With pure differential privacy, even the strongest attacker with arbitrary background information has limited ability to make inferences on the unknown record(s). With approximated differential privacy, this guarantee holds with a high chance, while failure of privacy preserving happens with probability at most $\delta$ (informally called ``all-bets-are-off''). In practice, $\delta$ should be taken significantly small, such as $\Theta(n^{-2})$.

While approximate DP is a relaxation of pure DP, some other relaxations of pure DP also exists, such as zero-concentrated differential privacy (zCDP) \cite{bun2016concentrated} and R\'enyi Differential Privacy (RDP) \cite{mironov2017renyi}. These relaxations do not have such semantic meanings as approximate DP, but they are shown to stand between pure and approximate DP: they provide weaker protection than pure DP, but stronger protection than approximated DP, for any given $\delta>0$. In this paper, we will focus on R\'enyi Differential Privacy.

\subsection{R\'enyi Differential Privacy}

Define $Z=\frac{\Pr[\mathcal M(D)\in S]}{\Pr[\mathcal M(D')\in S]}$ as the privacy loss random variable, instead of requiring it always inside range $[-\epsilon, \epsilon]$ as pure DP, R\'enyi differential privacy (RDP) constraints its expectation by R\'enyi divergence.

\begin{definition}[$(\alpha, \epsilon)$-R\'enyi Differential Privacy (RDP)]
\cite{mironov2017renyi} Given a real number $\alpha\in(1, +\infty)$ and privacy parameter $\epsilon\geq0$, a randomized mechanism (algorithm) $\mathcal M$ satisfies $(\alpha, \epsilon)$-RDP if for every pair of neighboring datasets $D\sim D'$, the R\'enyi $\alpha$-divergence between $\mathcal M(D)$ and $\mathcal M(D')$ satisfies
\begin{equation}
    D_{\alpha}[\mathcal M(D) \| \mathcal M(D')] \leq \epsilon
\end{equation}
\end{definition}

That is, the privacy parameter $\epsilon$ bounds the moment $\alpha$ of the R\'enyi divergence $D_\alpha$, which is defined as

\begin{definition}[R\'enyi Divergence]
For probability distributions $\mathcal M(D)$ and $\mathcal M(D')$ over a set $\Omega$, and let $\alpha\in (1, +\infty)$. Then R\'enyi $\alpha$-divergence is 
\begin{equation}
    D_\alpha(\mathcal M(D) \| \mathcal M(D')) := \frac{1}{\alpha-1}\log\mathbb E_{x\sim \mathcal M(D')}\bigg[ \bigg(\frac{P_{\mathcal M(D)}(x)}{P_{\mathcal M(D')}(x)}\bigg)^\alpha\bigg]
\end{equation}
\end{definition}

One method to achieve RDP is through the Gaussian mechanism: when a query $q(D)$ is taken over the dataset, the Gaussian mechanism adds a Gaussian noise $\gamma\sim \mathcal N(0, \sigma^2\textbf{I}_k)$, and release perturbed $q(D) + \gamma$.

\begin{lemma}[Gaussian Mechanism]
\cite{mironov2017renyi} Let $q: \mathcal D^n\rightarrow \mathbb R^k$ be a vector-valued function over datasets. Let $\mathcal M$ be a mechanism releasing $q(D) + \gamma$ where $\gamma\sim \mathcal N(0, \sigma^2\textbf{I}_k)$, then for any $D\sim D'$ and any $\alpha\in(1, +\infty)$,
\begin{equation}
    D_\alpha(\mathcal M(D) \| \mathcal M(D')\leq \alpha\Delta_2^2(q) / (2\sigma^2)
\end{equation}
\end{lemma}

Gaussian mechanism relies on the $L_2$ sensitivity:

\begin{definition}[$L_2$ sensitivity]
Let $q: \mathcal D^n\rightarrow \mathbb R^k$ be a vector-valued function over datasets. The $L_2$ sensitivity of $q$, denoted as $\Delta_2(q)$, is defined as
\begin{equation}
    \Delta_2(q) = \sup_{D\sim D'}\|q(D) - q(D')\|_2
\end{equation}
\end{definition}
Therefore, when scale the variance $\sigma^2=\alpha\Delta_2^2(q)/(2\epsilon)$, then $\mathcal M$ satisfies $(\alpha, \epsilon)$-RDP.

Gaussian mechanism makes the mechanism $\mathcal M$ satisfy $(\alpha, \epsilon)$-RDP for a series of $\alpha$, so we can use $\epsilon(\alpha)$ to denote the privacy $\epsilon$ under moment $\alpha$. In empirical risk minimization algorithms, it is common that the mechanism is taken over a randomized subsample of the dataset $B$, instead of the whole dataset $D$. Then, application Gaussian Mechanism on the subsample $B$ would satisfy $(\alpha, \epsilon(\alpha))$-RDP with respect to $B$. Due to the subsampling procedure, the mechanism would satisfy an amplified privacy with respect to the whole dataset $D$, as given by the following lemma:

\begin{lemma}[RDP for subsampled mechanism]
\cite{wang2018subsampled} For a randomized mechanism $\mathcal M$ and a dataset $D\sim \mathcal D^n$, define $\mathcal M\circ\textsc{subsample}$ as (1) subsample without replacement $m$ datapoints from the dataset (denote $q=m/n$ as sampling ratio); (2) apply $\mathcal M$ on the subsampled dataset as input, then if $\mathcal M$ satisfies $(\alpha, \epsilon(\alpha))$-RDP with respect to the subsample for all integers $\alpha>2$, then the new randomized mechanism $\mathcal M\circ\textsc{subsample}$ satisfies $(\alpha, \epsilon'(\alpha))$-RDP with respect to $D$, where
\begin{equation}
\begin{split}
    \epsilon'(\alpha) \leq \frac{1}{\alpha - 1} \log\big( & 1  + q^2{\alpha\choose 2} \min\big\{ 4(e^{\epsilon(2)}-1), 2e^{\epsilon(2)} \big\} \\
     & + \sum_{j=3}^\alpha q^j {\alpha\choose j} 2e^{(j-1)\epsilon(j)}\big)
\end{split}
\end{equation}
\end{lemma}

Similar as DP, RDP has below composition properties:

\begin{lemma}[RDP composition]
\cite{mironov2017renyi} For randomized mechanisms $\mathcal M_1$ and $\mathcal M_2$ applied on dataset $D$, if $\mathcal M_1$ satisfies $(\alpha, \epsilon_1)$-RDP and $M_2$ satisfies $(\alpha, \epsilon_2)$-RDP, then their composition $\mathcal M_1\circ \mathcal M_2$ satisfies $(\alpha, \epsilon_1 + \epsilon_2)$-RDP.
\end{lemma}

RDP is said to provide stronger protection than approximate DP, due to below conversion to $(\epsilon, \delta)$-DP:

\begin{proposition}[RDP to $(\epsilon, \delta)$-DP]
\cite{mironov2017renyi} If $\mathcal M$ satisfies $(\alpha, \epsilon)$-RDP, then it satisfies $(\epsilon(\delta), \delta)$-DP for $\epsilon(\delta) \geq \epsilon + \frac{\log(1/\delta)}{\alpha-1}$.
\end{proposition}

Therefore, when evaluating our proposed algorithms, to compare with other algorithms which satisfies $(\epsilon, \delta)$-DP, we keep track of $(\alpha, \epsilon)$ pairs which our algorithm satisfies for a series of $\alpha$ values, then convert each pair into a $(\epsilon(\delta), \delta)$ pair it satisfies by Proposition 1, for a pre-defined small $\delta$, and choose the smallest $\epsilon(\delta)$ as the $(\epsilon, \delta)$-DP it satisfies to compare with other algorithms.

\subsection{Regularized Empirical Risk Minimization}
Many problems in machine learning can be formulated as
empirical risk minimization (ERM), which seek a
solution $x^*\in\Theta$ that minimizes an empirical loss on the
training data:
\begin{equation}
    x^* = \argmin_{x\in \Theta} F(x, D) := \argmin_{x\in \Theta} \frac{1}{n}\sum_{i=1}^n\ell(x, d_i)\,,
\end{equation}
where $\Theta$ is a parameter space, 
$\ell$ is a \emph{loss} function. To
prevent overfitting, it is common to add a (data-independent)
regularization term into the objective function, i.e. $\ell(x, d_i) = f(x,
d_i) + R(x)$. For $L_1$ regularization, $R(x) = \lambda \|x\|_1$. For
example, $L_1$ regularized logistic regression, one can fit the model
by solving 
\begin{equation}\label{formula_lr}
    x^* = \argmin_{x\in \Theta} \frac{1}{n}\sum_{i=1}^n \log(1+\exp(-l_ix^Ts_i)) + \lambda \|x\|_1
\end{equation}
Recall that each datum $d_i=(s_i, l_i)$ as feature vector $s_i$ and
label $l_i$. However, due to that many optimization algorithms assume
the loss function to be doubly differentiable, it cannot be directly
used on $L_1$ regularization problems. In this paper, we make the
following assumptions on the loss function: 
\begin{itemize}
    \item \textbf{Convexity} Both the data-dependent function $f$ and regularization term $R$ are convex.
    \item \textbf{Differentiability} The non-regularized data-dependent function $f$ is continuously differentiable with respect to $x$.
    \item \textbf{Bounded gradient} There exists a constant $C>0$ such that $\|\nabla f(x, d)\|_2\leq C$ for all $x\in \Theta$ and $d\in \mathcal D$. Usually it is satisfied by preprocessing the data to ensure the feature $s_i$ of each data $d_i$ lies inside a ball of some radius $r$, or directly clip the $L_2$ norm of individual gradient by a threshold $C$.
\end{itemize} 

\subsection{Alternating Direction Method of Multipliers}

The Alternating Direction Method of Multipliers (ADMM) algorithm was proposed decades ago, and has recently been widely used to solve optimization problems in machine learning \cite{boyd2011distributed}. Consider the optimization problem 
\begin{equation}
\begin{split}
    & \text{minimize} \quad f(x) + h(z) \\
    & \text{subject to} \quad Ax + Bz = c
\end{split}
\end{equation}
where $f:\mathbb R^n \rightarrow \mathbb R$, $g:\mathbb R^m\rightarrow \mathbb R$, $A\in\mathbb R^{p\times n}$, $B\in\mathbb R^{p\times m}$, and $c\in\mathbb R^p$. ADMM forms the augmented Lagrangian of the problem:
\begin{equation}
    L_\rho(x,z,y) := f(x) + h(z) + y^T(Ax+Bz-c) + \frac{\rho}{2} \|Ax+Bz-c\|^2_2
\end{equation}
where $x,z$ are called the \emph{primal} variables, $y\in \mathbb R^p$ is called the \emph{dual} variable, and $\rho>0$ is a pre-selected \emph{penalty} parameter.

ADMM algorithm solves the optimization problem by alternating the iterations below
\begin{align}
& \text{$x$-minimization step: } x^{k+1} \leftarrow \argmin_x L_\rho(x, z^k, y^k) \label{formula_x_min}\\
& \text{$z$-minimization step: } z^{k+1} \leftarrow \argmin_z L_\rho(x^{k+1}, z, y^k) \\
& \text{dual variable update: } y^{k+1} \leftarrow y^k + \rho (Ax^{k+1} + Bz^{k+1} - c)
\end{align}
Therefore, $x$ and $z$ are updated in an alternating fashion, and separating minimization over $x$ and $z$ into two steps can make the otherwise hard-to-solve optimization problem solvable in a sequential manner.

\subsection{Stochastic ADMM}

One variant of ADMM, stochastic ADMM (sADMM), was proposed by \cite{ouyang2013stochastic} and tested on $L_1$ regularized linear regression (LASSO). This variant was proposed based on the observation that, for ADMM problems, usually one of $f(x)$ and $h(z)$ is data-dependent, and it is both expensive and unnecessary to exactly solve its minimization step for each iteration. To be specific, let $f$ be data-dependent, and $h$ be data-independent, then the optimization problem becomes $f(x, D) + h(z)$, and sADMM approximate $L_\rho$ by \emph{approximated} augmented Lagrangian $\hat L_\rho$, defined at iteration $k$ as
\begin{equation}
\begin{split}
     \hat{L}_{\rho}(x, z, y) := f(x^k) + \langle \nabla f(x^k, B_k), x \rangle + \frac{\|x - x^k\|^2_2}{2\eta^k} \\
     + h(z) + y^T(Ax+Bz-c) + \frac{\rho}{2} \|Ax+Bz-c\|^2_2 
\end{split}
\end{equation}
where $B_k$ is a portion of the data accessed at iteration $k$, and $\eta^k$ is the learning rate at iteration $k$. After this approximation of $L_\rho$ by $\hat L_\rho$, one can derive an exact solution for each $x$-minimization step in (\ref{formula_x_min}), instead of solving a computationally expensive ERM problem. 

For $L_1$ regularized ERM, let $h(z)$ be the regularization term $R(z)=\lambda \|z\|_1$, the constraint $Ax+Bz=c$ reduces to $x=z$, then by taking derivative of $\hat L_\rho(x, z^k, y^k)$ and set to zero, one get
\begin{equation}\label{formula_xupdate}
    x^{k+1} \leftarrow \frac{1}{\rho+1/\eta^k}(-\nabla f(x, B_k) - y^k + \rho z^k + x^k / \eta^k)
\end{equation}
as the exact solution to minimize $\hat L_{\rho}(x, z^k, y^k)$, and
\begin{equation}\label{formula_yupdate}
    y^{k+1} \leftarrow y^k + \rho(x-z)
\end{equation}
to update the dual variable $y$.

\section{Algorithm}
\label{sec:algorithm}

In this section we propose the main algorithms. We propose two sADMM based $L_1$ regularized classification algorithms, both satisfies R\'enyi differential privacy. One achieves privacy by gradient perturbation relying on randomized subsampling; the other is through model perturbation after each epoch relying on sensitivity calculation. Both algorithms assume a centralized computing: all training data were collected in a center, which performs the computation locally. This is because we assume the data is small-to-median sized, where $L_1$ regularization are usually applied on.

\subsection{R\'enyi differentially private subsampling algorithm}

Our subsampling private sADMM algorithm (ssADMM) is presented in Algorithm 1. This algorithm is inspired by the gradient perturbation technique proposed in \cite{abadi2016deep}, on differentially private stochastic gradient descent (DP-SGD).

Similar as DP-SGD, our ssADMM algorithm perturbs the mini-batch gradient by Gaussian noise right after gradient evaluation in line 6. However, Algorithm 1 differs from DP-SGD for the following aspects: (i) By utilizing ADMM, we are able separate gradient descent and $L_1$ regularization into two steps, so that pure gradient can be computed and perturbed in $x$-minimization step; for DP-SGD, proximal gradient has to be used to handle $L_1$ regularization; (ii) while DP-SGD suggest using constant learning rate, we proved that using decreasing step size in Algorithm 1 help accelerate convergence, as in Theorem 2 and numerical experiments; (iii) authors of DP-SGD proposed the moment accountant (MA) method to analyze the privacy loss, and convert to $(\epsilon, \delta)$-DP; we use the most recent RDP for subsampling mechanism, which is a more advanced technique to analyze privacy loss, and also easier to implement. 

\begin{algorithm}[H]
\caption{RDP subsampling sADMM $L_1$ regularized ERM algorithm (ssADMM)}

\begin{algorithmic}[1]

\State \textbf{Input}: Dataset $D = \{d_1,...,d_n\}$. Penalty parameter $\rho$, mini-batch size $m$, total iterations $T$.

\State \textbf{Initialize}: primal variables $x^0, z^0$, dual variable $y^0$.

\For {iteration $k=0,1,...,T-1$}

\State Sample mini-batch $B_k$ from $D$ of size $m$.

\State $g_k \gets \frac{1}{m}\sum_{d_i \in B_k}\nabla f(x^k, d_i)$ \Comment{compute gradient}

\State $\tilde g_k \gets g_k + \gamma$ where $\gamma\sim N(0, \sigma^2\textbf{I}_p)$ \Comment{perturb gradient by Gaussian noise}

\State Compute $x^{k+1}$ by (\ref{formula_xupdate}) using $\tilde g_k$ \Comment{primal variable $x$}

\State Compute $z^{k+1}$ by (\ref{formula_zupdate}) \Comment{primal variable $z$}

\State Compute $y^{k+1}$ by (\ref{formula_yupdate}) \Comment{dual variable $y$}

\EndFor

\State \textbf{Output}: $x^T$

\end{algorithmic}
\label{alg_gp}

\end{algorithm}

Since the regularization is data-independent, it does not cause any privacy leak. Therefore, any (non-) smooth regularizers are applicable for Algorithm \ref{alg_gp}, with the same privacy guarantee. Since in this paper we use $L_1$ regularization as an example, for the $z$-minimization step, we utilize soft-thresholding technique from \cite{boyd2011distributed} to acquire the solution to minimize $L_\rho(x^{k+1}, z, y^k)$:
\begin{equation}\label{formula_zupdate} 
    z^{k+1} \leftarrow \mathcal S_{\frac{\lambda}{\rho}}(x^{k+1} + y^k/\rho)
\end{equation}
where soft-thresholding operator is defined as
\begin{equation}\label{formula_soft_threshold}
    \mathcal S_{t}(x)_i =
    \begin{cases}
    x_i - t & \text{ if } x_i > t \\
    x_i + t & \text{ if } x_i < -t \\
    0 & \text{ otherwise}
    \end{cases} 
\end{equation}
Similar technique has been used in \cite{ouyang2013stochastic} and \cite{wang2019differential}.

Another ADMM based algorithm proposed in \cite{huang2019dp} (DP-ADMM) also used gradient perturbation technique. Our method differed from theirs for the following aspects: (i) DP-ADMM is used for distributed learning, so that the training objective is assigned into multiple parties each holding a portion of the data, instead in ssADMM it is the data dependent loss and regularization that are separated; (ii) in DP-ADMM, each party is perturbing full gradient and transmit to the center, so that there is no privacy amplification effect, therefore although both algorithms solve optimization approximately, their privacy loss is higher than ours at each step. Our methods differ from the ADMM-objP method (DPLL in \cite{wang2019differential}) for the following aspect: (i) ADMM-objP perturb the training objective at each iteration, and use full gradient descent multiple times to acquire exact solution at each iteration, which is not as efficient as ours, since our method only access a portion of data once at each step; (ii) ADMM-objP guarantees privacy only if exact solution is acquired at each step, therefore the privacy guarantee is only theoretically true. The privacy guarantee of ssADMM is given by Theorem 1.

\begin{theorem}
Algorithm 1 is $(\alpha, \epsilon)$-RDP.
\end{theorem}

\begin{proof}
We first show the $L_2$ sensitivity of batch gradient $g_k$. Assume neighboring mini-batches $B_i$ and $B'_i$ differ by one record $d_s\in B$ and $d_s\in B'$, by Definition 4,
\begin{equation}
\begin{split}
    \Delta_2^k(g) = & \Delta_2[\frac{1}{m}\sum_{d_i\in B_k}\nabla f(x^k, d_i)] \\
    = & \sup_{B_k\sim B'_k}\|\frac{1}{m}\sum_{d_i\in B_k}\nabla f(x^k, d_i) - \frac{1}{m}\sum_{d_i\in B'_k}\nabla f(x^k, d_i)\|_2\\
    = & \frac{1}{m}\sup\|\nabla f(x^k, d_s) - \nabla f(x^k, d'_s)\|_2 \leq \frac{2C}{m}
\end{split}
\end{equation}

Let $\epsilon_k(\alpha) = \alpha(\Delta_2^k(g))^2 / 2\sigma^2$. So each iteration is $(\alpha, \epsilon_k(\alpha))$-RDP by Lemma 1, with respect to the batch $B_k$. Since $B_k$ is a randomized subsample of $D$, by Lemma 2, we can calculate $\epsilon'_k(\alpha)$ so that each iteration is $(\alpha, \epsilon'_k(\alpha))$-RDP with respect to $D$. Since the algorithm has run $T$ iterations, let $\epsilon = \sum_{k=0}^{T-1}\epsilon'_k(\alpha)$, by Lemma 4, Algorithm 1 is $(\alpha, \epsilon)$-RDP. 
\end{proof}

\begin{theorem}
If we choose $\eta^k = O(1/\sqrt{k})$, and train for $t$ iterations, then Algorithm 1 has the expected convergence rate of $O(1/\sqrt{t})$.
\end{theorem}

\begin{proof}
See proof in appendix.
\end{proof}

\subsection{R\'enyi differentially private model perturbation algorithm}

Our model perturbation private sADMM algorithm (mpADMM) is presented in Algorithm 2. Different from perturbing the gradients, this algorithm use the unperturbed gradients to do model calculation for a whole step, and keep track of the $L_2$ sensitivity of all data-dependent model vectors. After each epoch, Gaussian noises are injected into model vectors $x, y, z$, and total privacy $\epsilon$ is updated, according to sensitivity and $\sigma^2$. Due to it is difficult to calculate the sensitivity over multiple epochs, we perform output perturbation after each epoch. Therefore, this algorithm can be considered as multiple-time output perturbation algorithm.

\begin{algorithm}[H]
\caption{RDP model perturbation sADMM $L_1$ regularized ERM algorithm (mpADMM)}

\begin{algorithmic}[1]

\State \textbf{Input}: Dataset $D = \{d_1,...,d_n\}$. Penalty parameter $\rho$, total epochs $T$.

\State \textbf{Initialize}: primal variables $x^0, z^0$, dual variable $y^0$.

\For {epoch $k=0,1,...,T-1$}

\State $g_k \gets \frac{1}{n}\sum_{d_i \in D}\nabla f(x^k, d_i)$ \Comment{compute gradient}

\State Compute $x^{k+1}$ by (\ref{formula_xupdate}) \Comment{primal variable $x$}

\State Compute $z^{k+1}$ by (\ref{formula_zupdate}) \Comment{primal variable $z$}

\State Compute $y^{k+1}$ by (\ref{formula_yupdate}) \Comment{dual variable $y$}

\State Sample $\gamma_1, \gamma_2, \gamma_3 \sim N(0, \sigma^2\textbf{I}_p)$

\State $x^{k+1} \gets x^{k+1} + \gamma_1$, $y^{k+1} = y^{k+1} + \gamma_2$, $z^{k+1} = z^{k+1} + \gamma_3$ \Comment{perturb the model}

\EndFor

\State \textbf{Output}: $x^T$

\end{algorithmic}

\end{algorithm}

To calculate the sensitivity, since unperturbed batch gradient is used here, after one epoch, all primal and dual variables are data-dependent. Assume neighboring datasets $D$ and $D'$ differ at position $s$: $d_s\in D$ and $d'_s\in D'$. We define $\delta_x:=x - (x')$ where $x$ and $(x')$ are primal variables evaluated on $D$ and $D'$, respectively, after one epoch. Also, define $\delta_z^k$ and $\delta_y^k$ similarly. Then, after epoch $k$,
\begin{equation}
\begin{split}
    \delta_x^{k+1} = & x^{k+1} - (x')^{k+1} \\
    = & \frac{1}{\rho+1/\eta^k}(-\frac{1}{n}\sum_{d_i \in D}\nabla f(x^k, d_i) - y^k + \rho z^k + x^k / \eta^k) - \\
    & \frac{1}{\rho+1/\eta^k}(-\frac{1}{n}\sum_{d_i \in D'}\nabla f(x^k, d_i) - y^k + \rho z^k + x^k / \eta^k) \\
    = & (\nabla f(x^k, d'_s) -\nabla f(x^k, d_s)) / n(1 + \eta^{k+1}\rho)
\end{split}
\end{equation}
Consider when the soft-thresholding operator $\mathcal S_t$ (\ref{formula_soft_threshold}) applied on two vectors $w$ and $w'$, and compare $\mathcal S_t(w) - \mathcal S_t(w')$ with $w - w'$ element-wise:
\begin{itemize}
    \item If $w_i$ and $w'_i$ are of different signs, applying $\mathcal S$ on $w_i$ and $w'_i$ would bring them closer, therefore $|\mathcal S_t(w_i) - \mathcal S_t(w'_i) | < |w_i - w'_i|$;
    \item If $w_i$ and $w'_i$ are of the same sign, without loss of generality, let $|w_i| \leq |w'_i|$. One can easily observe that
    \begin{itemize}
        \item If $t\leq |w_i| \leq |w'_i|$, then $|\mathcal S_t(w_i) - \mathcal S_t(w'_i) | = |(|w_i|-t) - (|w'_i|-t)| = |w_i - w'_i|$;
        \item If $|w_i| < t < |w'_i|$, then $|\mathcal S_t(w_i) - \mathcal S_t(w'_i) | = |0 - (|w'_i|-t)| < |w_i - w'_i|$ since $t < |w'_i|$;
        \item If $|w_i| \leq |w'_i| \leq t$, then $|\mathcal S_t(w_i) - \mathcal S_t(w'_i) | = 0 \leq |w_i - w'_i|$;
    \end{itemize}
\end{itemize}
For vectors $u, v$, we can use $u \preccurlyeq v$ to represent $|u_i| < |v_i|$ and $u_i, v_i$ have the same sign, for each index $i$. Obviously $u\preccurlyeq v$ indicates $\|u\|_2 \leq \|v\|_2$. In either case above, we have $|\mathcal S_t(w_i) - \mathcal S_t(w'_i)| \leq |w_i - w'_i|$, and sign preserves (or becomes zero), so $\mathcal S_t(w) - \mathcal S_t(w') \preccurlyeq w - w'$ for any threshold $t$. Therefore,
\begin{equation}
\begin{split}
    \delta_z^{k+1} = & z^{k+1} - (z')^{k+1} \\
    = & \mathcal S_{\frac{\lambda}{\rho}}(x^{k+1} + y^k / \rho) - \mathcal S_{\frac{\lambda}{\rho}}((x')^{k+1} + y^k / \rho) \\
    \preccurlyeq & x^{k+1} + y^k/\rho - ((x')^{k+1} + y^k/\rho) = \delta_x^{k+1}
\end{split}
\end{equation}
and
\begin{equation}
    \begin{split}
    \delta_y^{k+1} & = y^{k+1} - (y')^{k+1} \\
    & = y^k + \rho(x^{k+1} - z^{k+1}) - \big(y^k + \rho((x')^{k+1} - (z')^{k+1})\big) \\
    & =\rho(\delta_x^{k+1} - \delta_z^{k+1}) \preccurlyeq \rho\delta_x^{k+1}
\end{split}
\end{equation}
The last $\preccurlyeq$ holds because $\delta_z^{k+1} \preccurlyeq \delta_x^{k+1}$, the subtraction by $\delta_z^{k+1}$ only pushes each element of $\delta_x^{k+1}$ towards zero. So we have below conclusions for sensitivities of $x, z, y$ after epoch $k$:
\begin{equation}
    \Delta_2^{k+1}(x) = \|\delta_x^{k+1}\|_2 \leq \frac{2C}{n(1+\eta^{k+1}\rho)}
\end{equation}
\begin{equation}
    \Delta_2^{k+1}(z) = \|\delta_z^{k+1}\|_2 \leq \|\delta_x^{k+1}\|_2 \leq \frac{2C}{n(1+\eta^{k+1}\rho)}
\end{equation}
\begin{equation}
    \Delta_2^{k+1}(y) = \|\delta_y^{k+1}\|_2 \leq \rho\|\delta_x^{k+1}\|_2 \leq \frac{2\rho C}{n(1+\eta^{k+1}\rho)}
\end{equation}

\begin{theorem}
Algorithm 2 is $(\alpha, \epsilon)$-RDP.
\end{theorem}

\begin{proof}
Let $\epsilon_{k+1, w}(\alpha) = \alpha(\Delta_2^{k+1}(w))^2 / 2\sigma^2$ for $w\in\{x, z, y\}$. By Lemma 1, each epoch is $(\alpha, \sum_{w\in\{x,z,y\}}\epsilon_{k+1, w}(\alpha))$-RDP, with respect to $D$. Since the algorithm has run $T$ epochs, by Lemma 4, let $\epsilon = \sum_{k=1}^{T}\sum_{w\in\{x,z,y\}}\epsilon_{k, w}(\alpha))$, then Algorithm 2 is $(\alpha, \epsilon)$-RDP.
\end{proof}

\section{Experimental Results}
\label{sec:experiment}

In this section we will present our experimental results on both real and simulated datasets. We will first show performance of classification on two real datasets, then show performance of both classification and feature selection on a synthetic dataset.

\subsection{ERM models}
We perform our experiments on $L_1$ regularized logistic regression and huberized SVM. The objective function of logistic regression is in (\ref{formula_lr}). For huberized SVM, the objection function is
\begin{equation} \label{formula_svm}
    F(x, D) := \frac{1}{n}\sum_{i=1}^n \ell_{\text{huber}}(l_ix^Ts_i) + \lambda \|x\|_1
\end{equation}
where
\begin{equation}
    \ell_\text{huber}(z) :=
    \begin{cases}
    0 & \text{ if } z > 1 + h \\
    \frac{1}{4h}(1+h-z)^2 & \text{ if } |1-z| \leq h \\
    1-z & \text{ otherwise}
    \end{cases} 
\end{equation}
is the huberized hinge loss (we set $h=0.5$ in all experiments).

\subsection{Baselines}
Many differentially private ERM algorithms cannot be applied to $L_1$ regularized classification, such as ObjPert \cite{chaudhuri2011differentially}, \cite{kifer2012private}, OutPert \cite{zhang2016dynamic}, PVP and DVP \cite{zhang2017efficient}, PSGD \cite{wu2017bolt}, and RSGD \cite{chen2019renyi}. 
Therefore, we compare our proposed algorithms with these baselines: DP-SGD \cite{abadi2016deep}, DP-ADMM  \cite{wang2019differential}, ADMM-objP \cite{huang2019dp}, and Non-Private approach. 

DP-SGD performs stochastic gradient descent with Gaussian perturbation. (Although their paper proposed moment accountant approach to analyze the privacy leak, we use Lemma 2 to analyze as we do on ssADMM, since it gives tighter bound on $\epsilon$.)
For DP-SGD, when the algorithm requires taking gradient on $f(x^k, B_k) + \lambda\|x^k\|_1$, we use the proximal gradient technique 
\begin{equation}
x^{k+1}\gets \mathcal S_{\lambda\eta^k}[x^k - \eta^k \nabla f(x^k, B_k)]
\end{equation}
to update $x^{k+1}$, as suggested in \cite{duchi2009efficient} and \cite{combettes2011proximal}. 
DP-ADMM is a distributed learning version of ADMM, where each party transfers perturbed primal variables to the center, and the center draw a consensus of the parties then transfer primal and dual variable back to each party.
ADMM-objP is an ADMM version of the objective perturbation algorithm. At each iteration, the trainer optimize a perturbed unregulated objective function, therefore although the algorithm satisfies pure $\epsilon$-DP, in practice it is not really differentially private due to the objective function can only be approximately solved. According to their paper, we apply gradient descent enough times and assume the optimization problem is exactly solved at each iteration. 

The DP-SVRG algorithm presented in \cite{wang2017differentially} can also be applied on non-smooth regularizers, but we have implemented and found that, due to the extra privacy budget required to spent on perturbing the full gradient, with the high privacy range $(\epsilon\leq 1)$, if we choose a large $\sigma^2$, the perturbed full gradient cannot help as a control variant to fasten the training, but actually slows down the minimization of empirical loss; if we choose a small $\sigma^2$, the privacy budget accumulates too fast and exceed our range in a few iterations. Therefore we have dropped this algorithm in our comparisons.

\subsection{Datasets and Pre-proessing}

Two real datasets on human subjects were used in our study: (i) the Adult dataset \cite{chang2011libsvm} was generated from 1994 US Census, with $n=48,842$, $p=124$, and the frequency of the majority label is 0.761; (ii) the IPUMS-BR dataset \cite{ruggles2015integrated} was extracted from IPUMS data, with $n=38,000$, $p=53$, and the frequency of the majority label is 0.507. 

To test the performance on feature selection, we created a synthetic dataset with many irrelevant features, using similar strategy as in \cite{wang2019differential}. To be specific, we generate a 100-dimension data $s_i \sim \mathcal N(0_{100}, \Sigma)$ where $\Sigma_{i,j}=0.5^{|i-j|}$. Let $x$ be the true model, defined as $x_{1:10} = (0.5, 1, 1.5, 2, 2.5, 3, 3.5, 4, 4.5, 5)$, $x_{11:20}=-x_{1:10}$, and $x_{21:100}=(0,...,0)$. For the label of each row $l_i$, we sample the Bernoulli distribution with $P(l_i=1)=1 / (1 + \exp(-x^Ts_i + \iota))$, where $\iota\sim \mathcal N(0, 1)$ is a random noise. Therefore, to predict $l_i$, $s_i$ contains 20 relevant features and 80 irrelevant features. We generate 40,000 samples to constitute one dataset, the frequency of the majority label is 0.500. We only perform logistic regression on simulated data, since it is usually used for attribute selection.

We did 10-fold cross validation on each experiment for each algorithm, and due to randomness from noisy injection, we repeat each fold 10 times and report average classification accuracy and objective value on testing data. For the simulated data, we generated 10 datasets using the simulation strategy, and report the average performance.

An intercept is added into each dataset. All numerical attributes are re-scaled into [0, 1] by Min-Max scalar. For the algorithms requiring feature vector to have bounded $L_2$ norm, we normalize to make $\|x_i\|\leq 1$ for $i=1,...,n$.

\subsection{Parameter setting}

We keep $\delta=10^{-8}$ for all experiments. For those algorithms satisfying RDP, we choose the best conversion to $(\epsilon, \delta)$-DP. In non-private settings, model users usually train a series models with different candidates of regularization coefficient $\lambda$, and select the one with highest testing performance. However, this process is data-dependent, therefore in private settings we cannot take a ``best performing'' coefficient for granted. Instead, we performed two group of experiments by two frequently using coefficients: low regularization with $\lambda=0.0001$ and high regularization with $\lambda=0.001$.

For ssADMM and DP-SGD, we set mini-batch size $m=\sqrt{n}$. We choose $\eta^k = \eta^0 / h$ where $h$ is the current expected epoch (we consider every $n/m$ iterations as one expected epoch), since we find this schedule has the best performance for both algorithms, compare to a constant learning rate, or a decreasing one at a rate of $O(1/\sqrt{k})$. After tuning on the simulated data, we set penalty term $\rho=0.25$ for ssADMM and $\rho=0.5$ for mpADMM. For mpADMM, we use a constant learning rate $\eta$.
For DP-ADMM, we assume there are 2 parties, each holding half of the data. (If there is only one party, DP-ADMM will reduce to DP-SGD with sampling ratio=1.) 
For ADMM-objP, at each iteration we optimize the perturbed objective function by full gradient descent running 20 epochs. Other parameters for DP-ADMM and ADMM-objP are set according to their paper.

\subsection{Classification Performance on Real Data}

\begin{figure*}[ht]
    \centering
    \begin{subfigure}[b]{0.245\textwidth}
        \includegraphics[width=\textwidth]{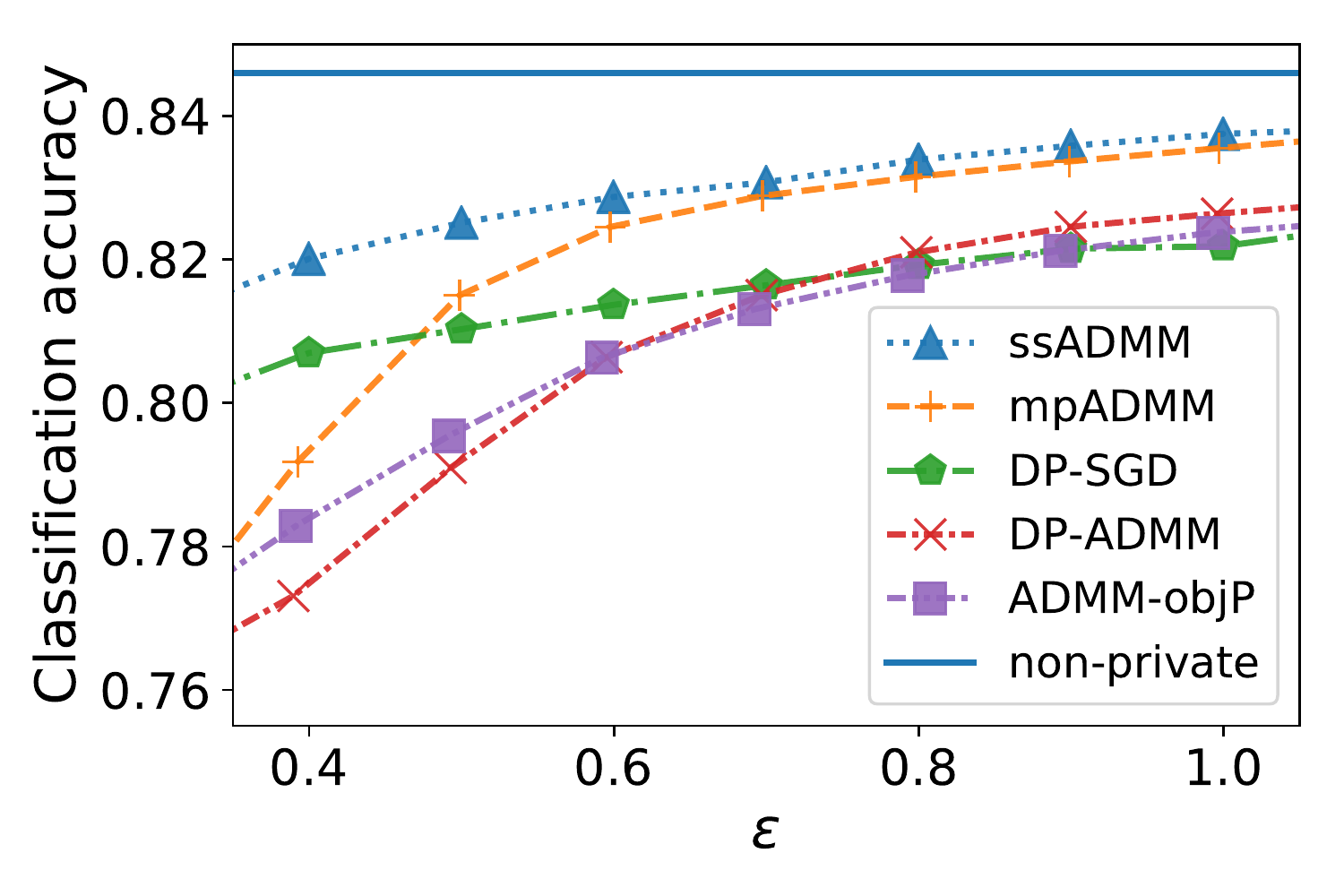}
        \includegraphics[width=\textwidth]{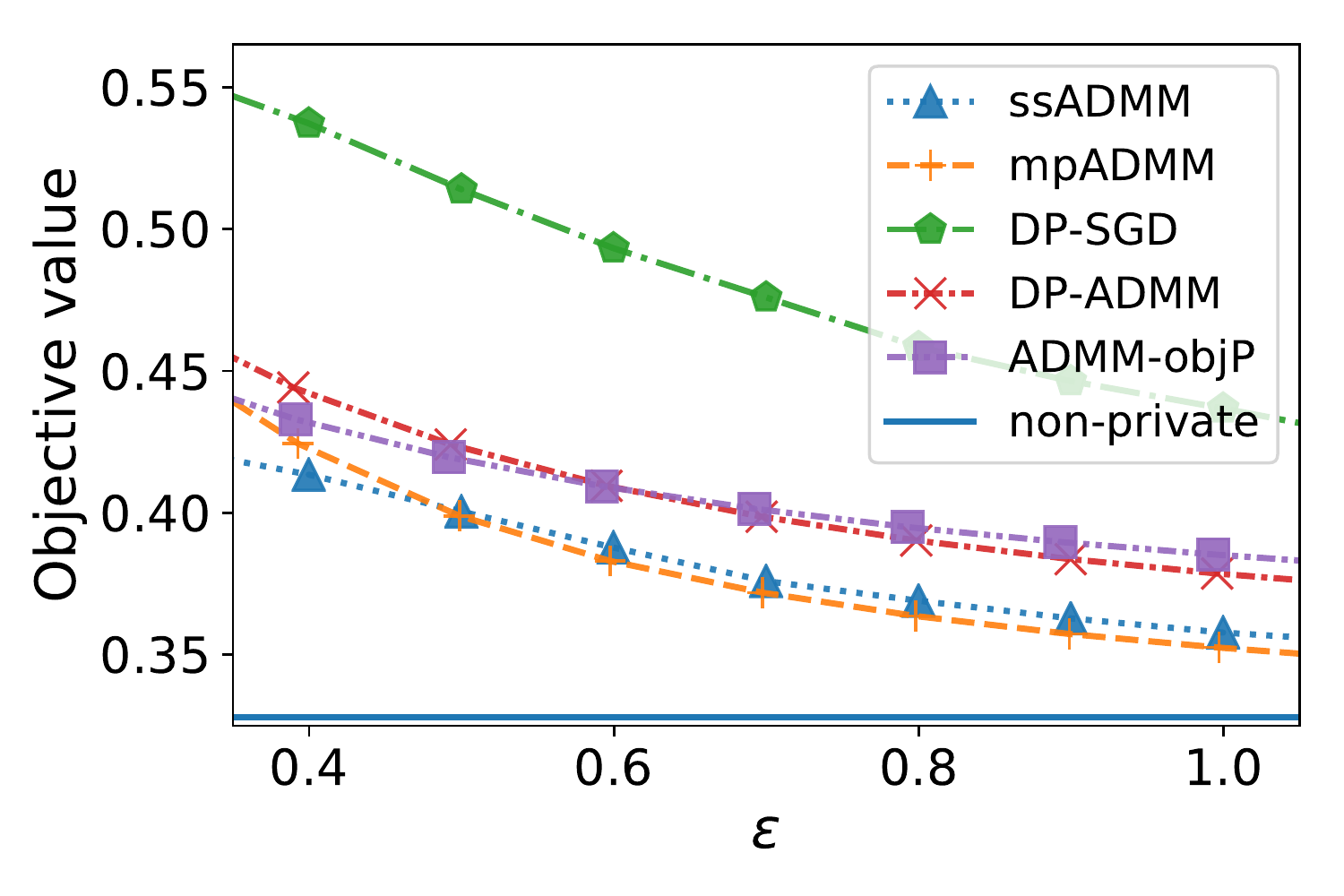}
        \caption{Adult $\lambda=0.0001$}
    \end{subfigure}
    \begin{subfigure}[b]{0.245\textwidth}
        \includegraphics[width=\textwidth]{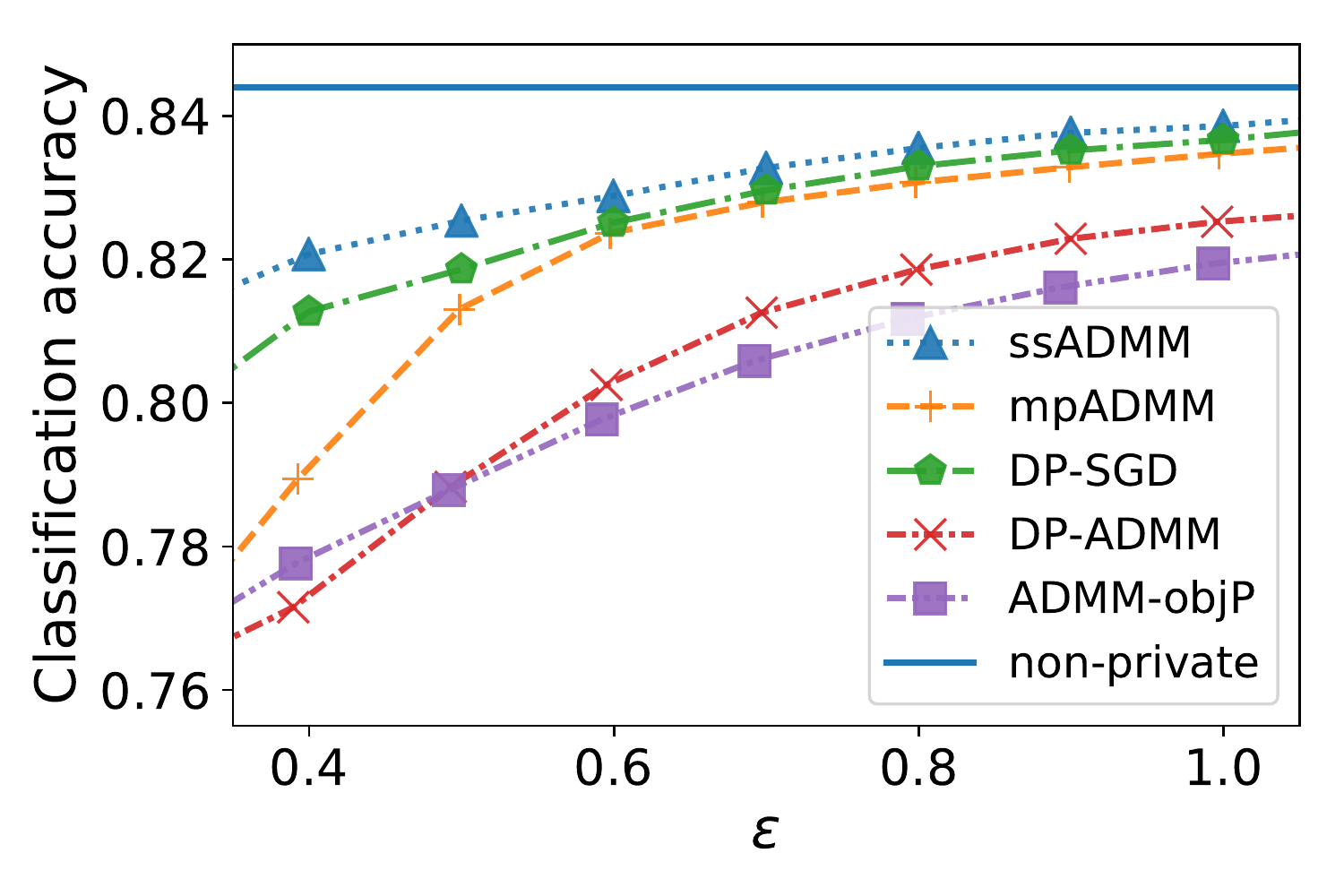}
        \includegraphics[width=\textwidth]{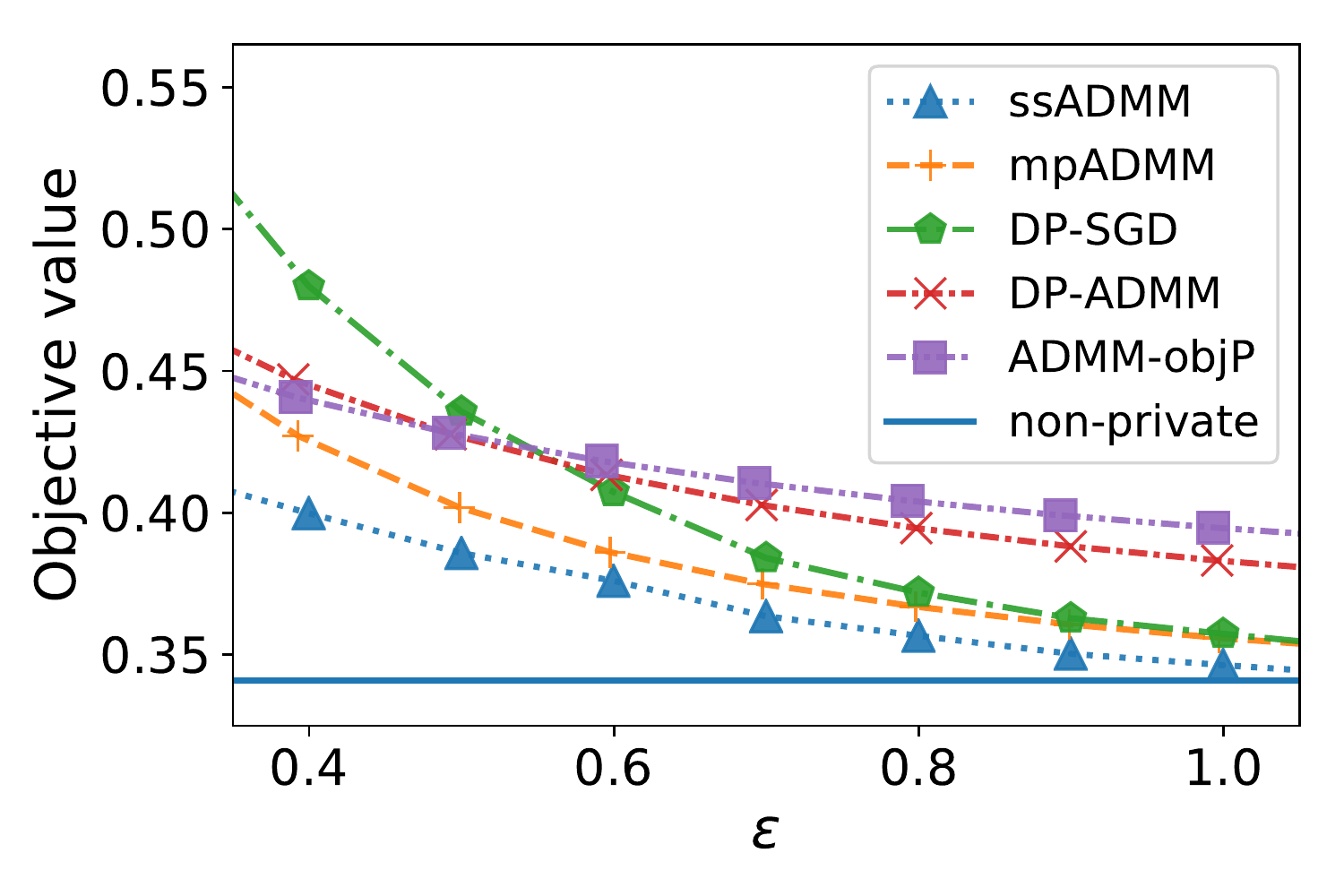}
        \caption{Adult $\lambda=0.001$}
    \end{subfigure}
    \begin{subfigure}[b]{0.245\textwidth}
        \includegraphics[width=\textwidth]{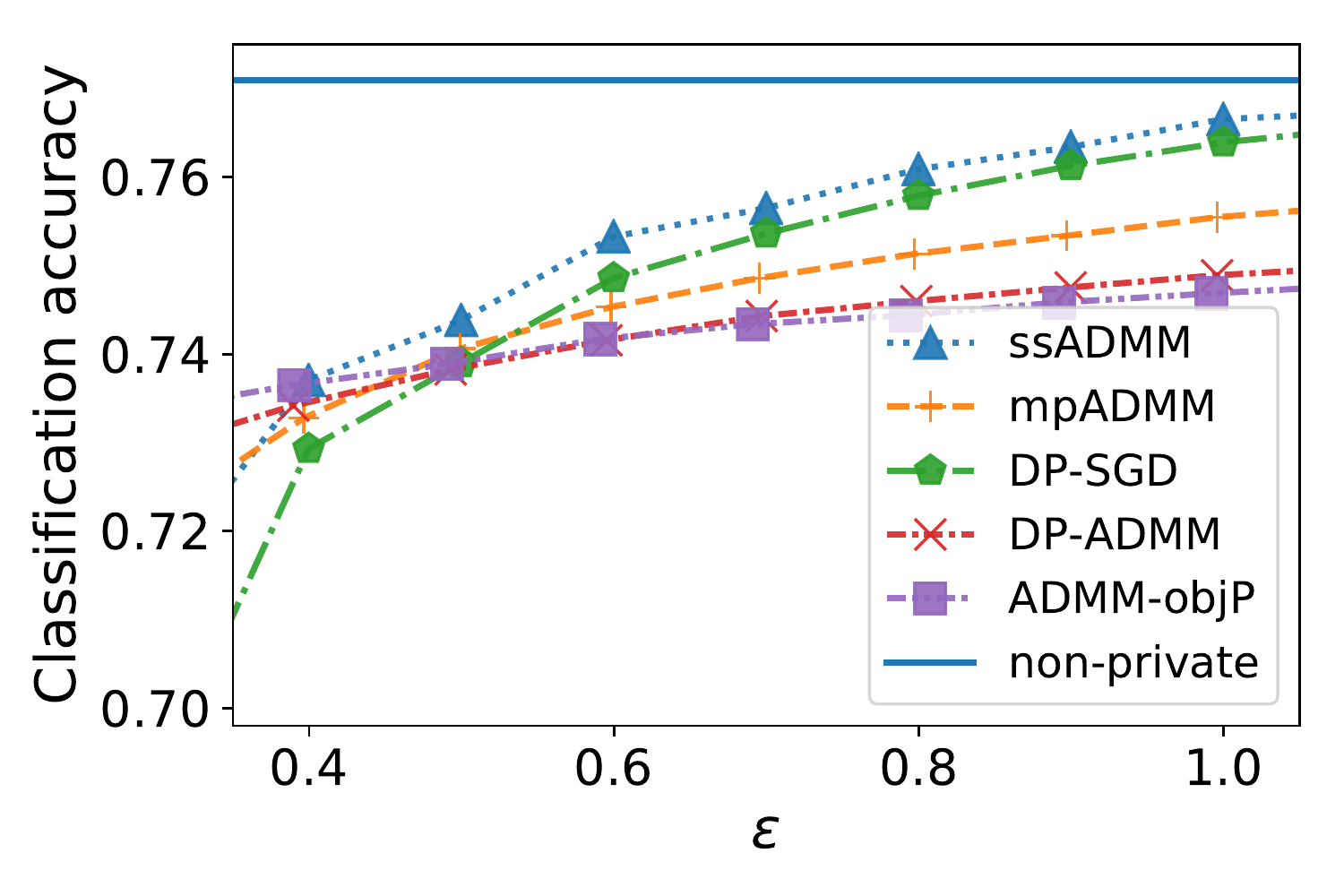}
        \includegraphics[width=\textwidth]{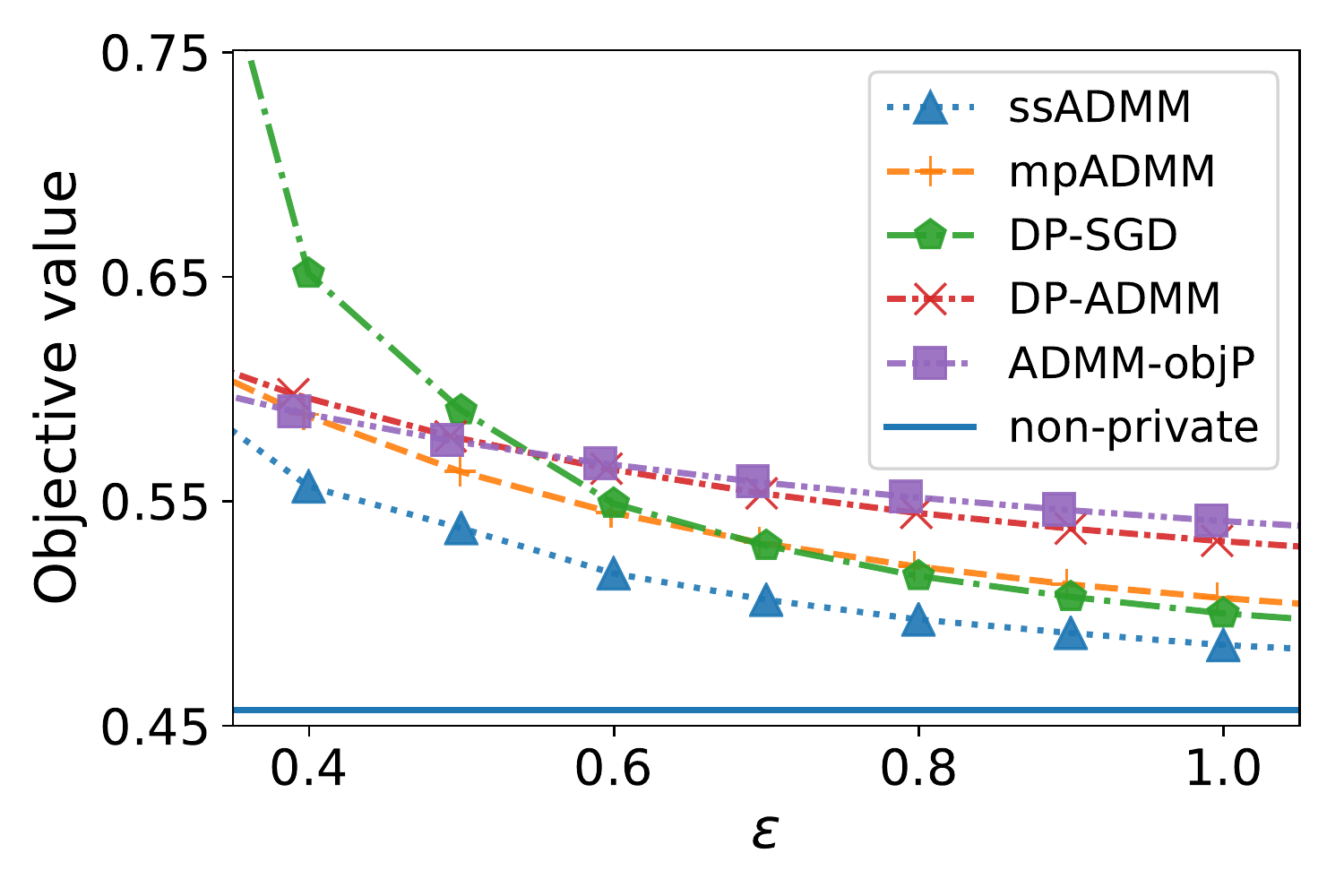}
        \caption{IPUMS-BR $\lambda=0.0001$}
    \end{subfigure}
    \begin{subfigure}[b]{0.245\textwidth}
        \includegraphics[width=\textwidth]{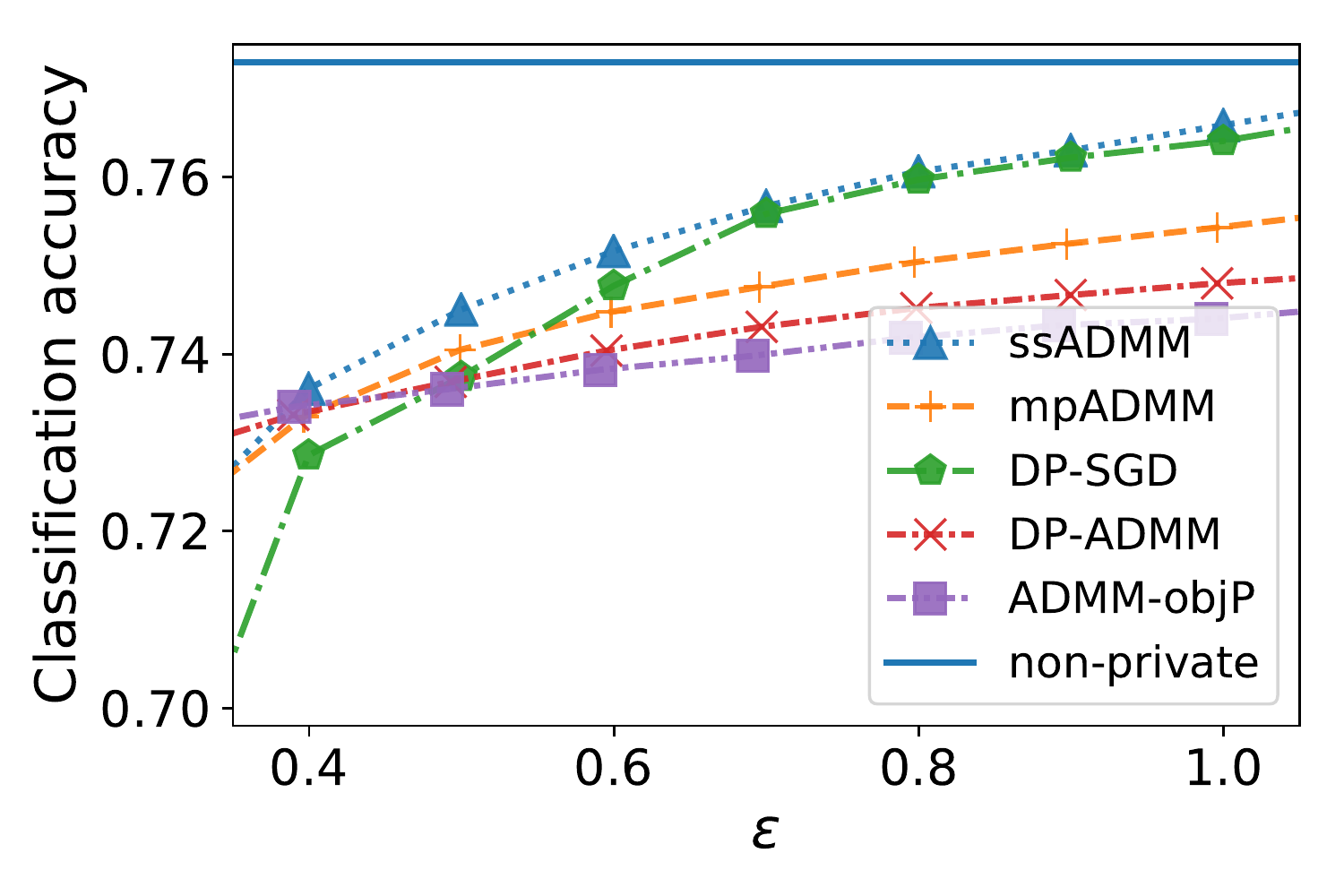}
        \includegraphics[width=\textwidth]{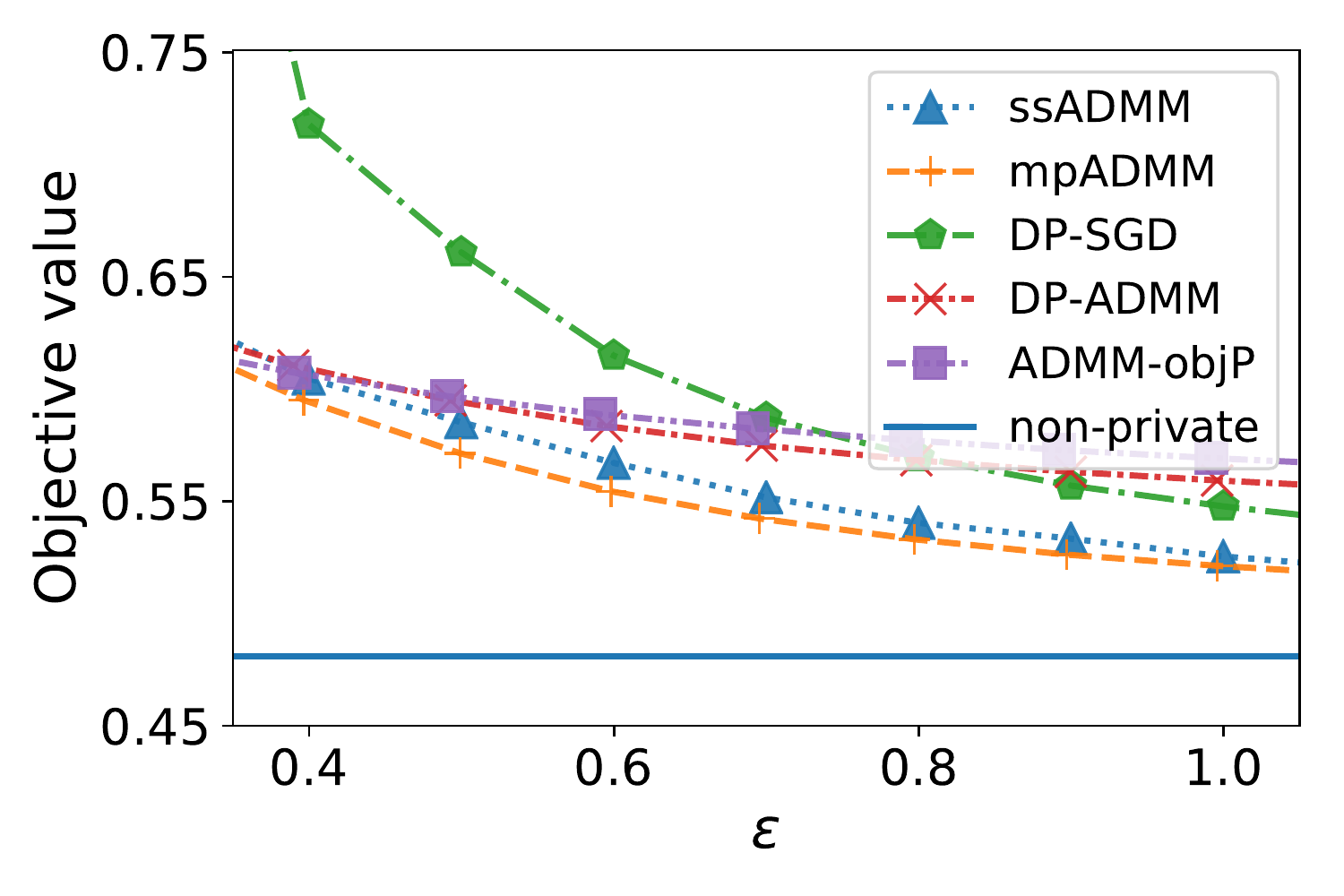}
        \caption{IPUMS-BR $\lambda=0.001$}
    \end{subfigure}
    \caption{Logistic regression result by $\epsilon$ (Top: Classification accuracy; Bottom: Objective value)}
    \label{figure_lr}
\end{figure*}

\begin{figure*}[ht]
    \centering
    \begin{subfigure}[b]{0.245\textwidth}
        \includegraphics[width=\textwidth]{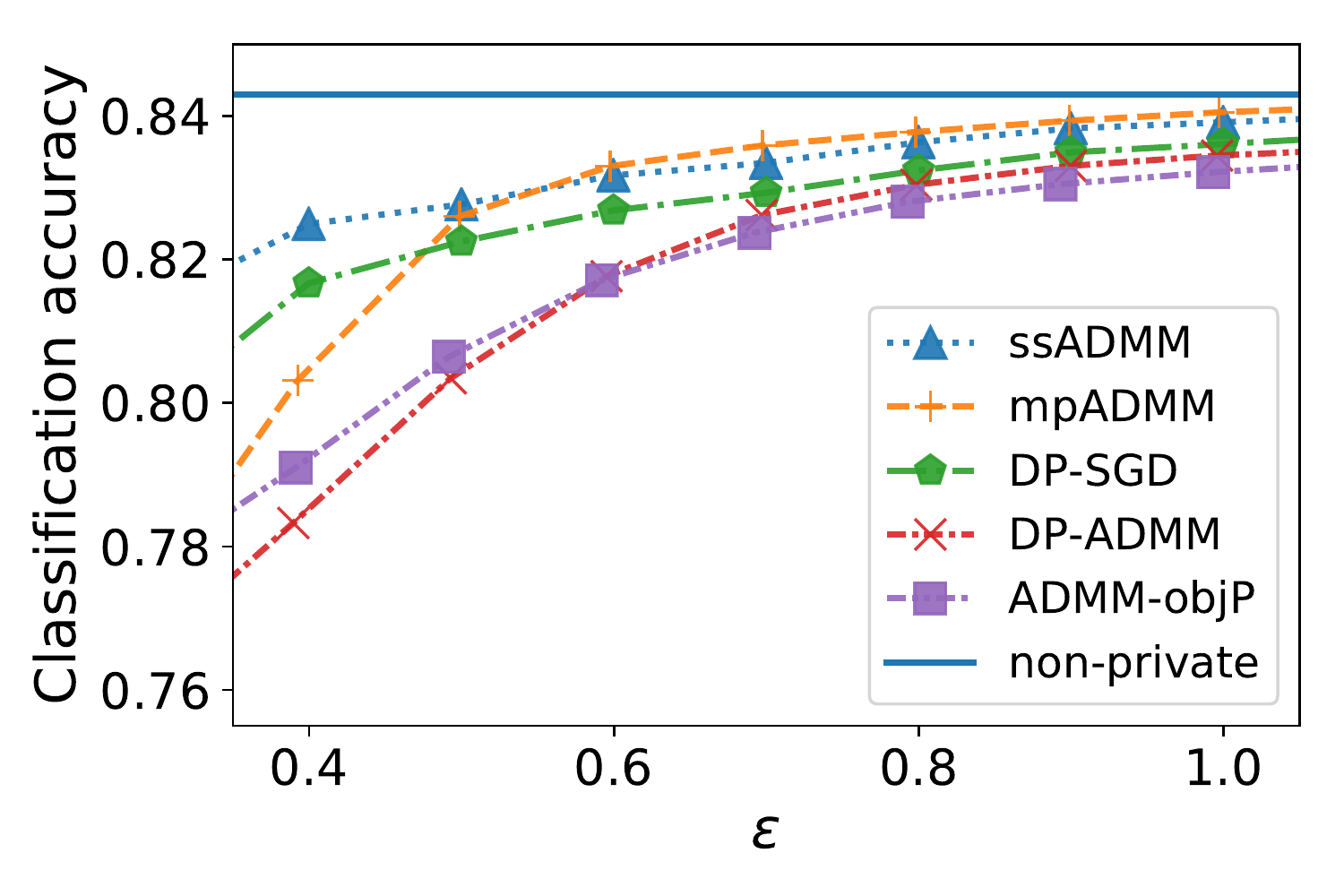}
        \includegraphics[width=\textwidth]{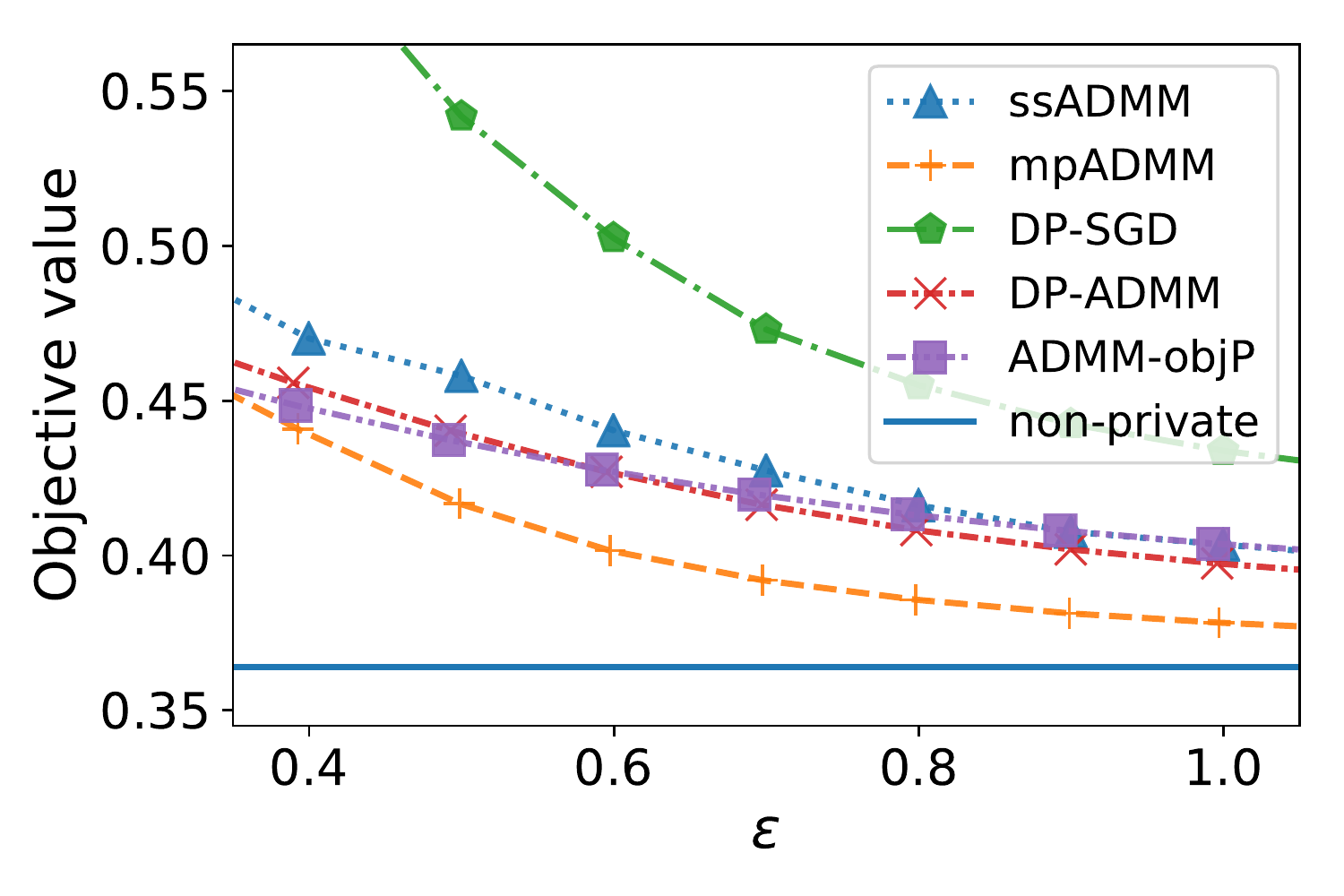}
        \caption{Adult $\lambda=0.0001$}
    \end{subfigure}
    \begin{subfigure}[b]{0.245\textwidth}
        \includegraphics[width=\textwidth]{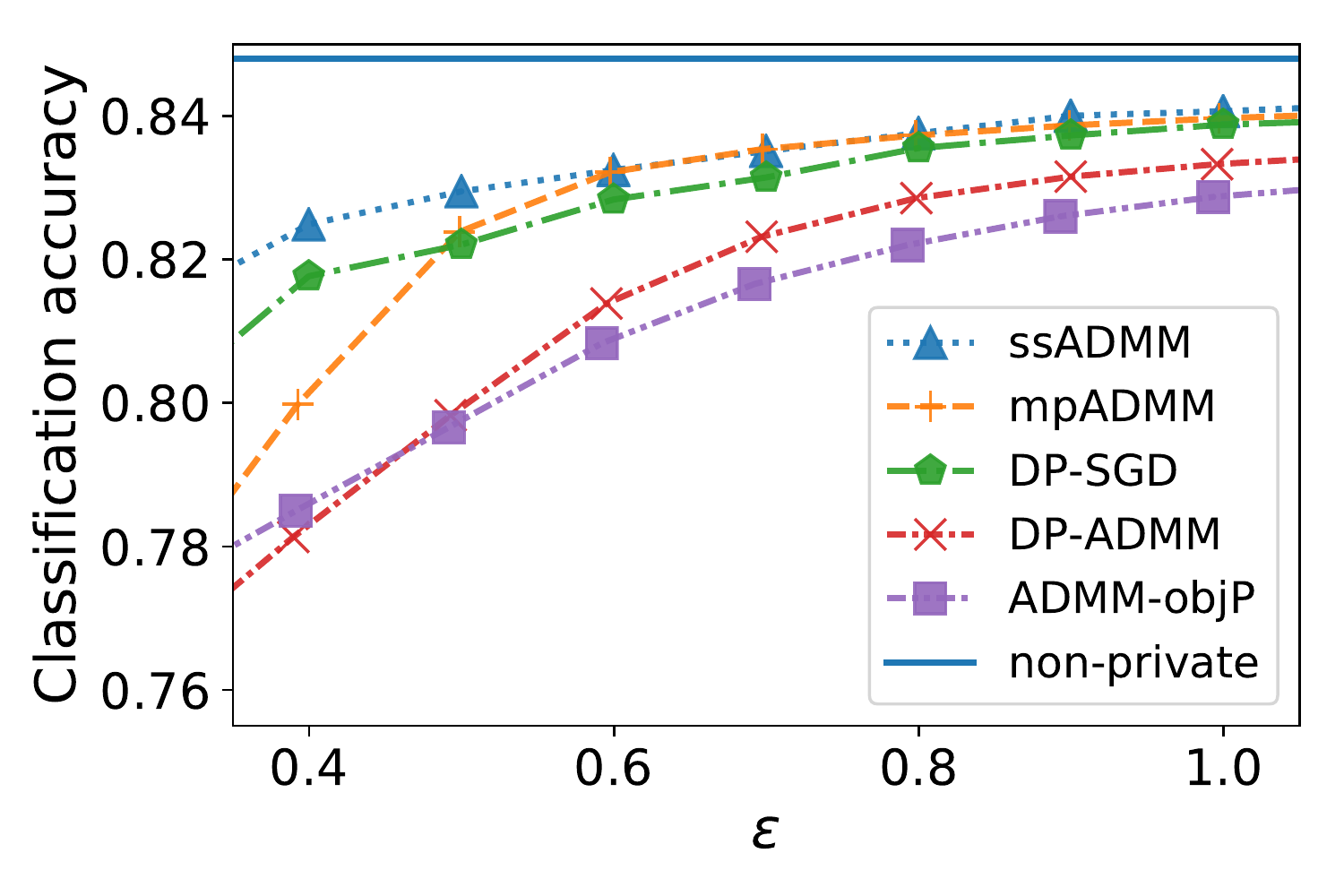}
        \includegraphics[width=\textwidth]{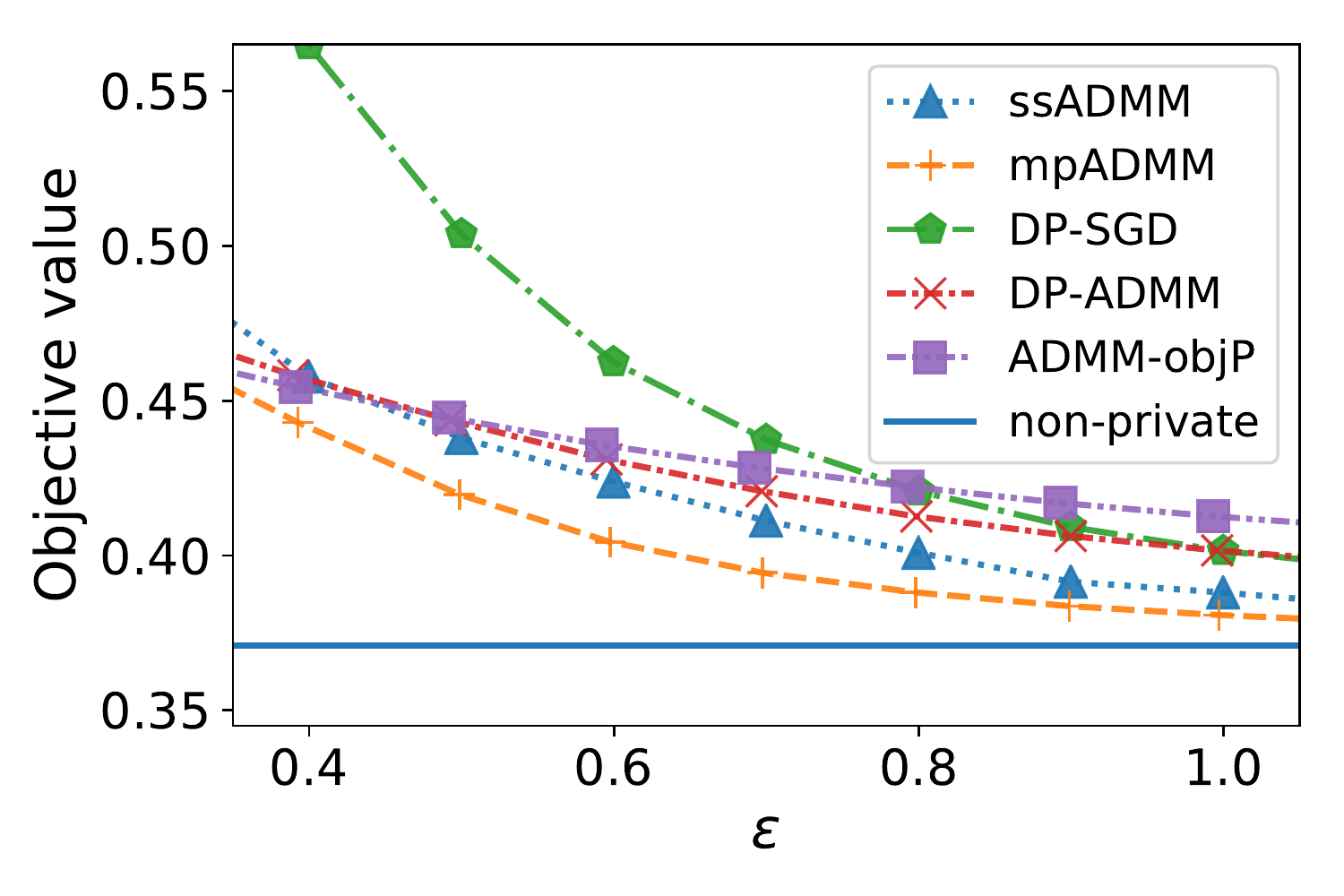}
        \caption{Adult $\lambda=0.001$}
    \end{subfigure}
    \begin{subfigure}[b]{0.245\textwidth}
        \includegraphics[width=\textwidth]{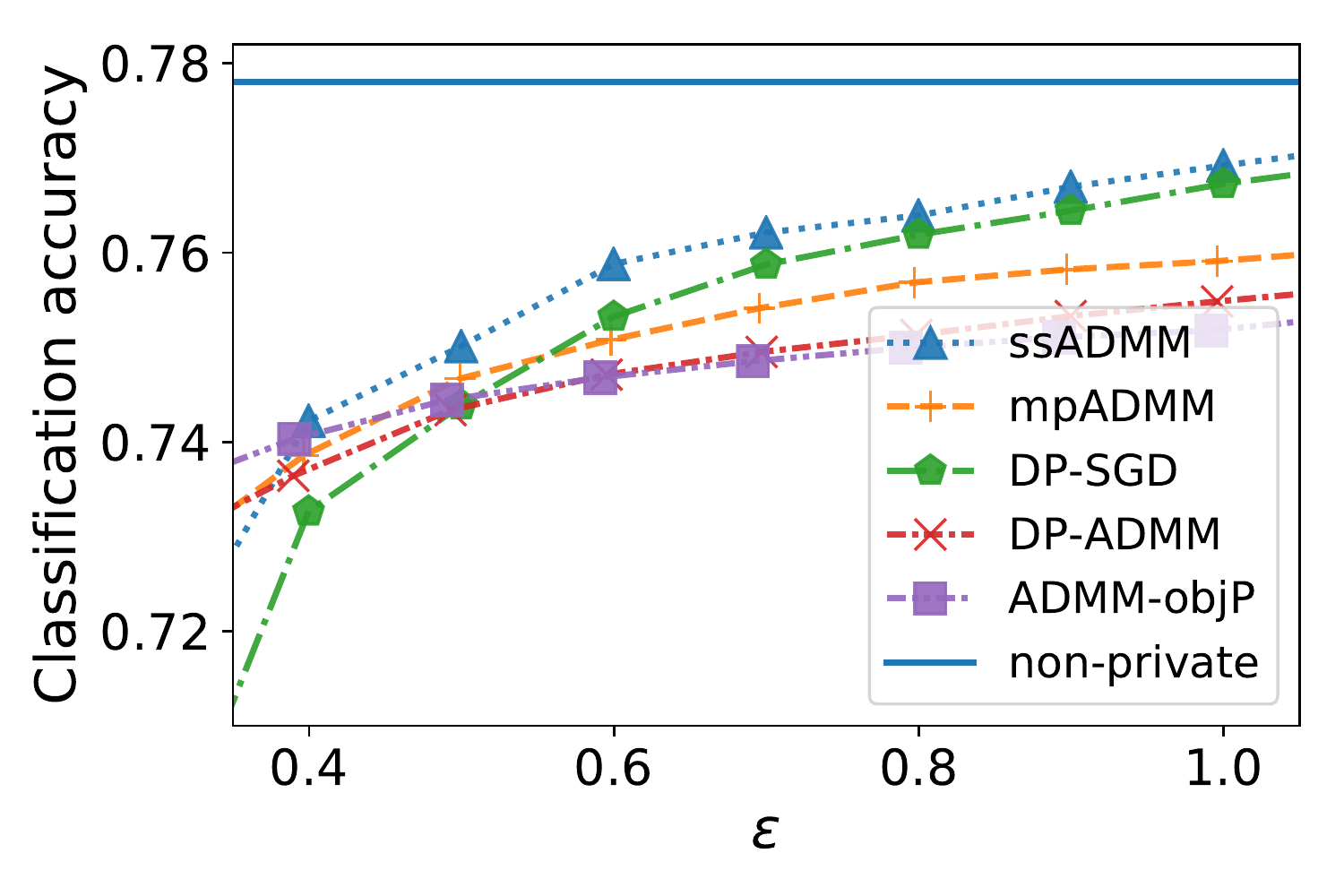}
        \includegraphics[width=\textwidth]{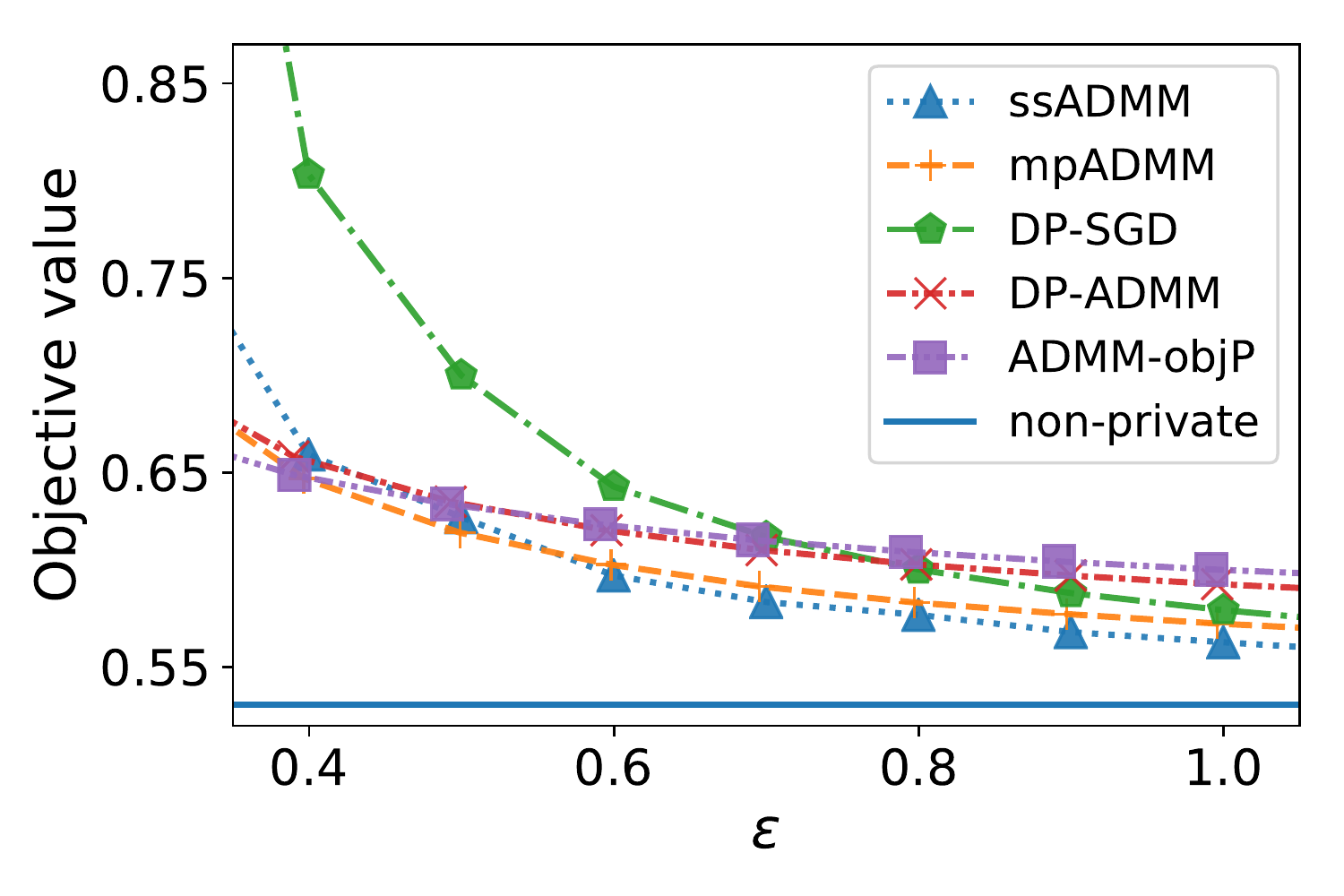}
        \caption{IPUMS-BR $\lambda=0.0001$}
    \end{subfigure}
    \begin{subfigure}[b]{0.245\textwidth}
        \includegraphics[width=\textwidth]{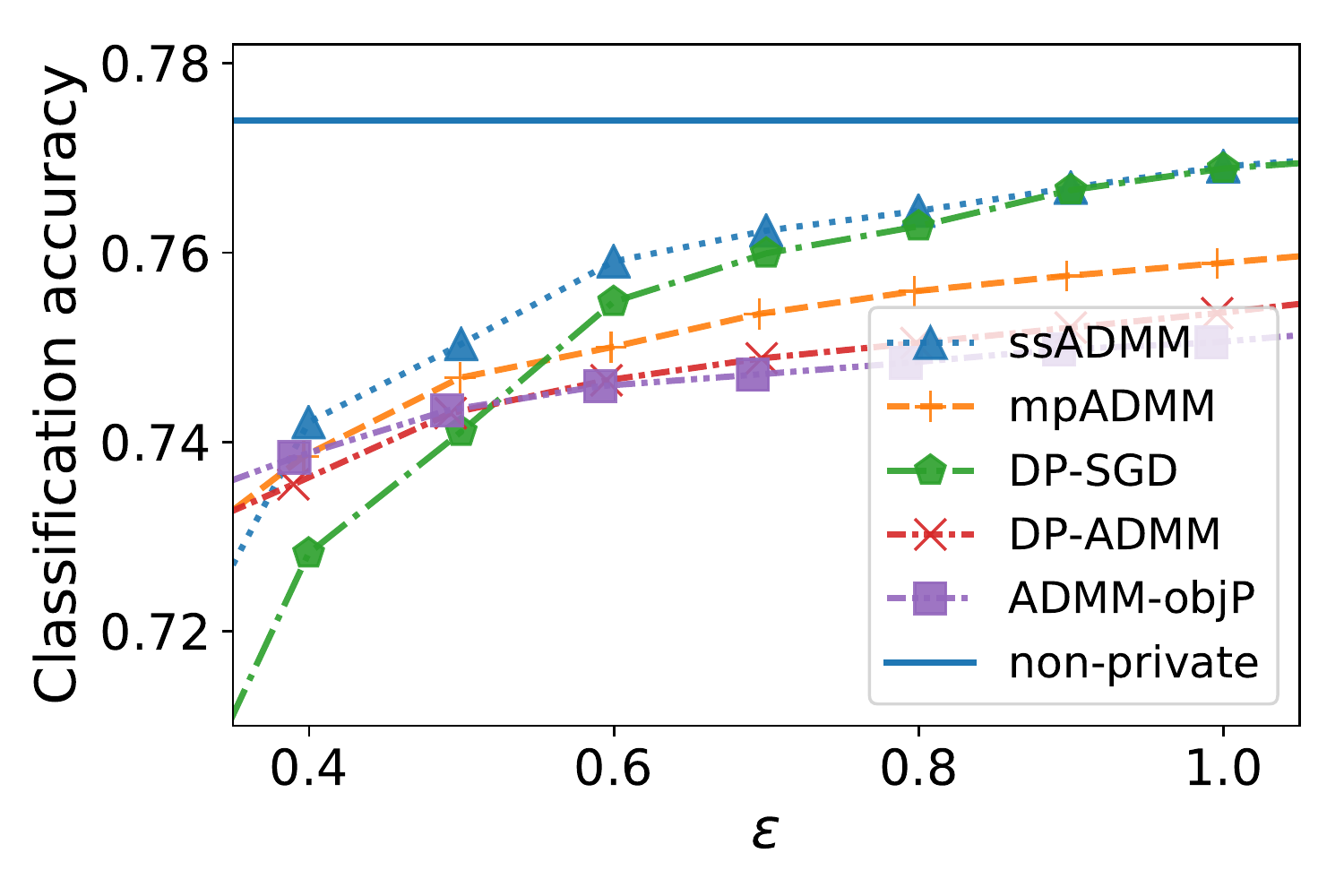}
        \includegraphics[width=\textwidth]{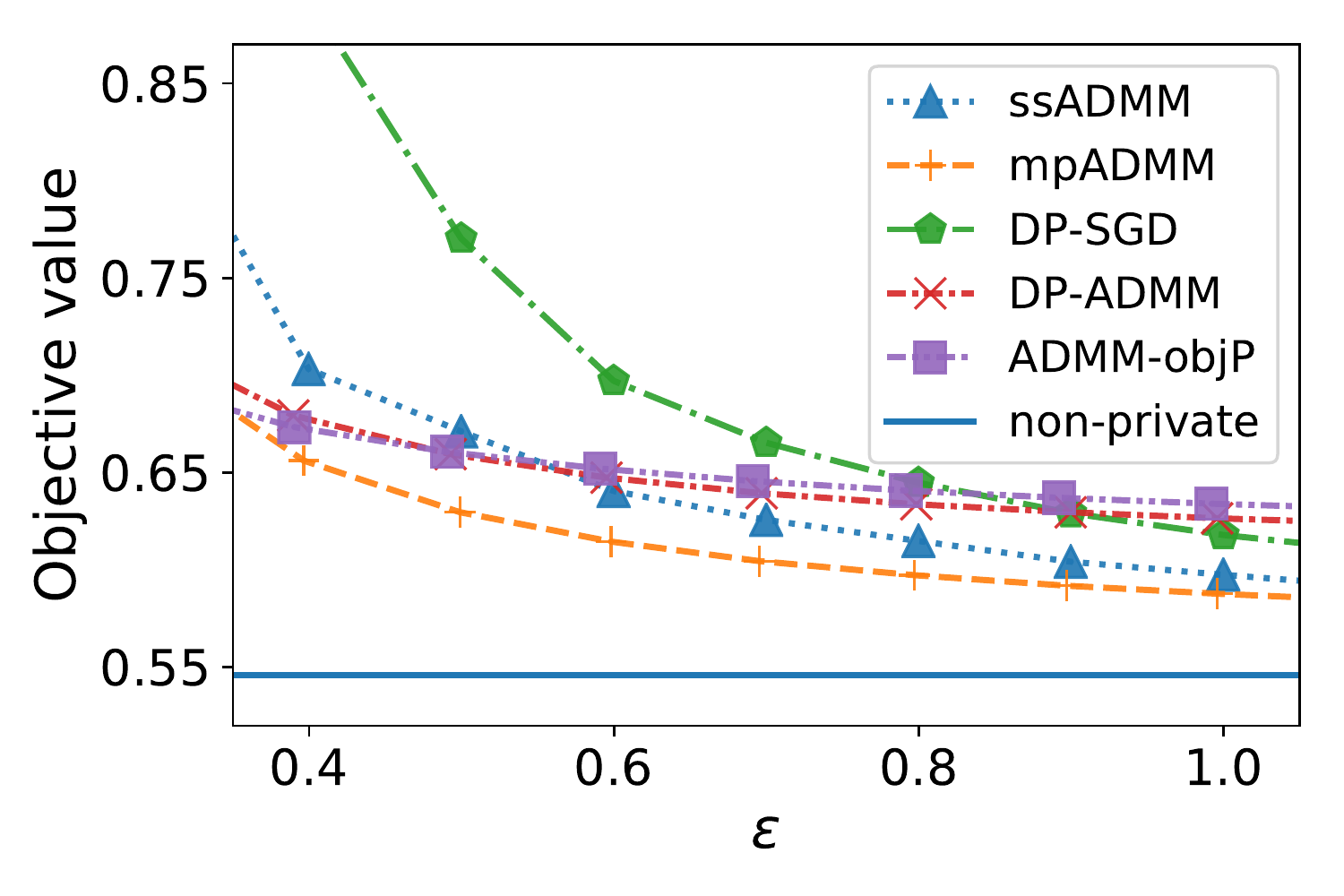}
        \caption{IPUMS-BR $\lambda=0.001$}
    \end{subfigure}
    \caption{Huberized SVM result by $\epsilon$ (Top: Classification accuracy; Bottom: Objective value)}
    \label{figure_svm}
\end{figure*}

Figure \ref{figure_lr} and Figure \ref{figure_svm} plots the testing data accuracy (top) and objective values (bottom) of the algorithms trading off with privacy parameter $\epsilon$, for $L_1$ regularized logistic regression and huberized SVM, respectively. We can see for classification accuracy, ssADMM outperforms other algorithms in most cases. This is in accordance with the experiment in \cite{ouyang2013stochastic} that sADMM outperforms proximal gradient in non-private setting. \cite{azadi2015auxiliary} also show that ADMM based algorithms are more robust to noisy data with outliers. Although DP-SGD has better classification accuracy than mpADMM in some cases, its objective value is usually higher. DP-ADMM and ADMM-objP can achieve high utility when $\epsilon$ gets high, but in our testing range of $\epsilon$, they cannot perform as good as other algorithms. mpADMM performs better in adult dataset than in IPUMS-BR dataset, probably because Adult dataset is more sparse compare to IPUMS-BR, due to it is binary transferred through one-hot encoding. And that model perturbation are more robust to data with irrelevant attributes is in accordance with our observations on the simulated data.

\subsection{Performance on Simulated Data}

\begin{figure*}[ht]
    \centering
    \begin{subfigure}[b]{0.245\textwidth}
        \includegraphics[width=\textwidth]{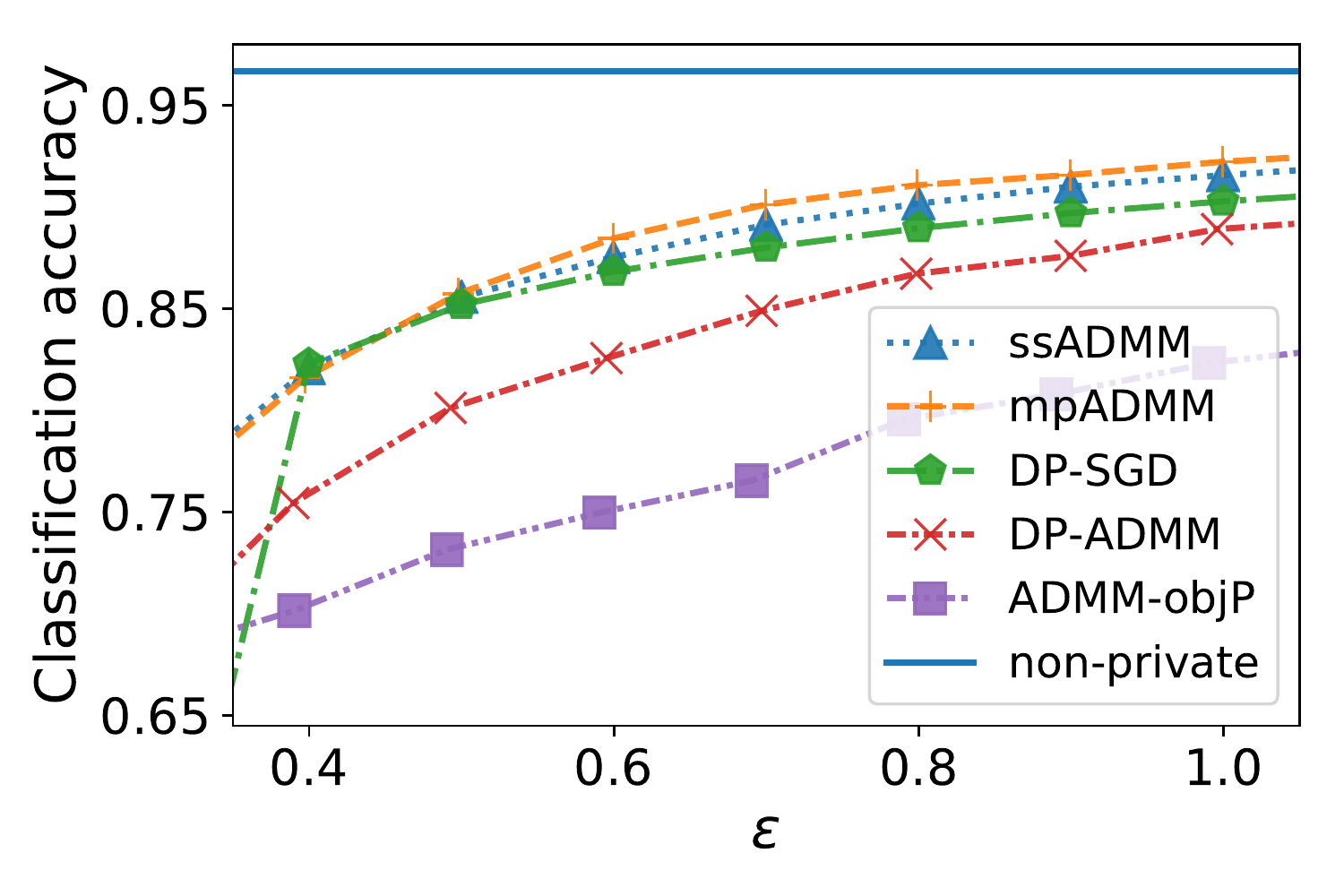}
        \caption{Accuracy $\lambda=0.0001$}
    \end{subfigure}
    \begin{subfigure}[b]{0.245\textwidth}
        \includegraphics[width=\textwidth]{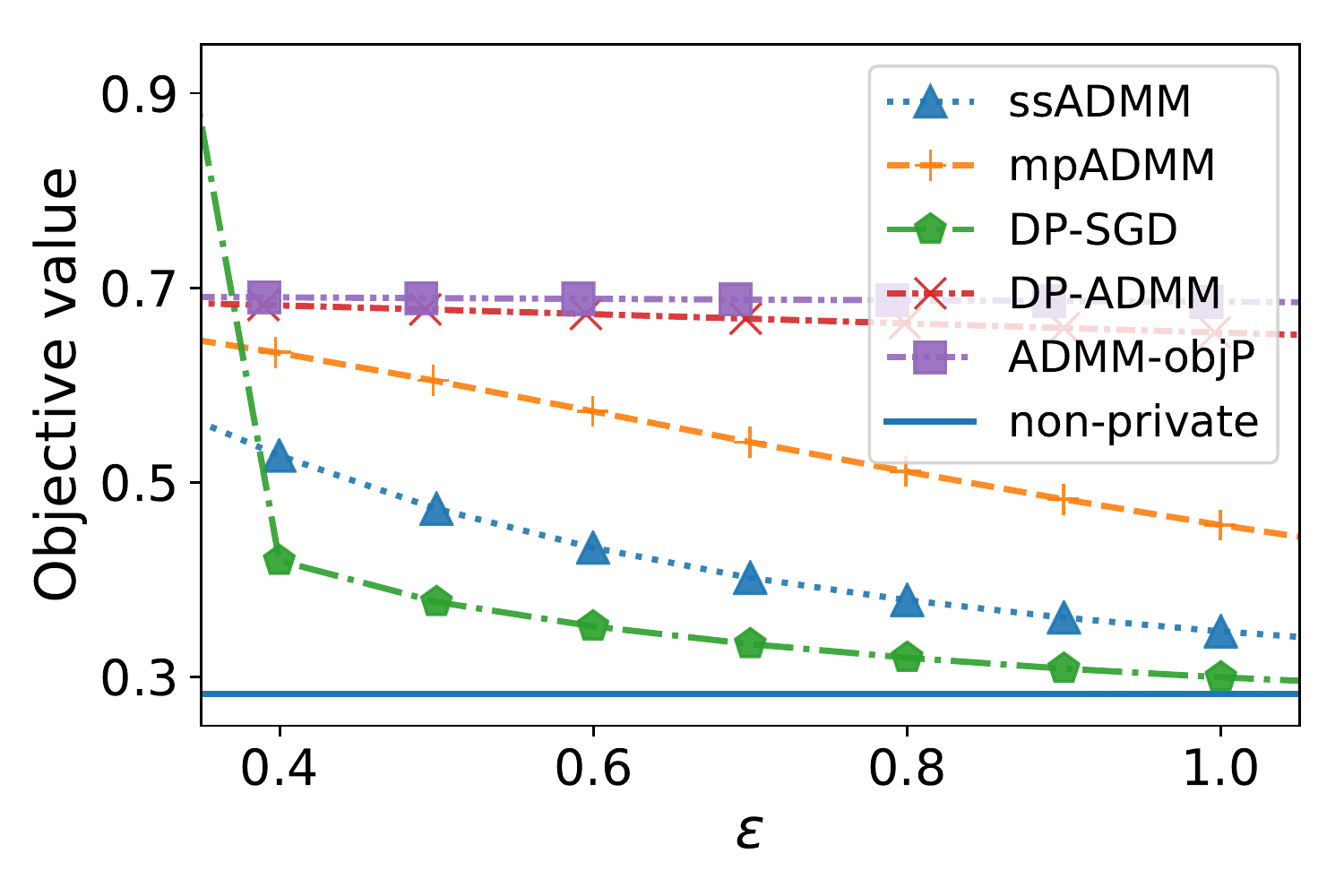}
        \caption{Objective value $\lambda=0.0001$}
    \end{subfigure}
    \begin{subfigure}[b]{0.245\textwidth}
        \includegraphics[width=\textwidth]{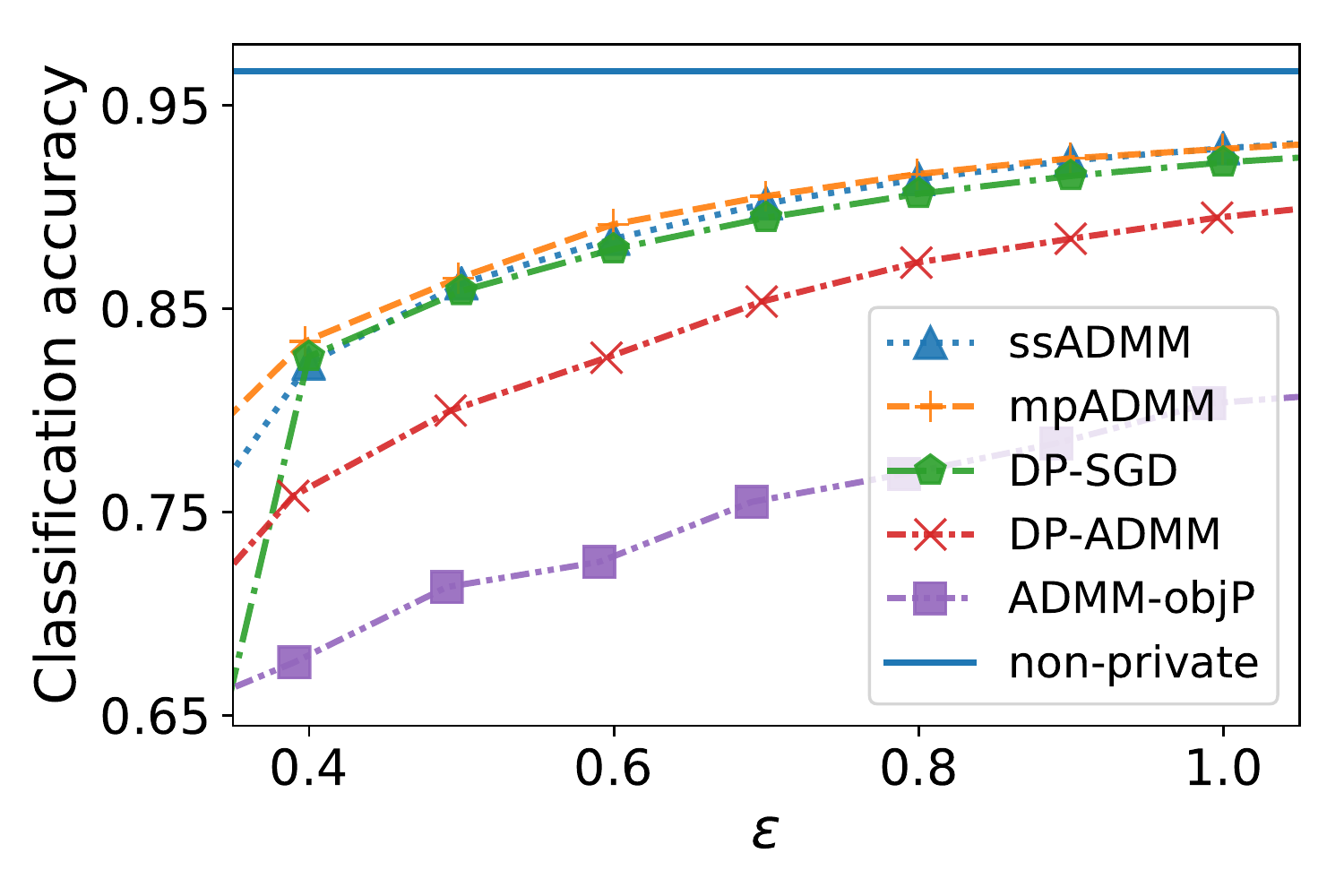}
        \caption{Accuracy $\lambda=0.001$}
    \end{subfigure}
    \begin{subfigure}[b]{0.245\textwidth}
        \includegraphics[width=\textwidth]{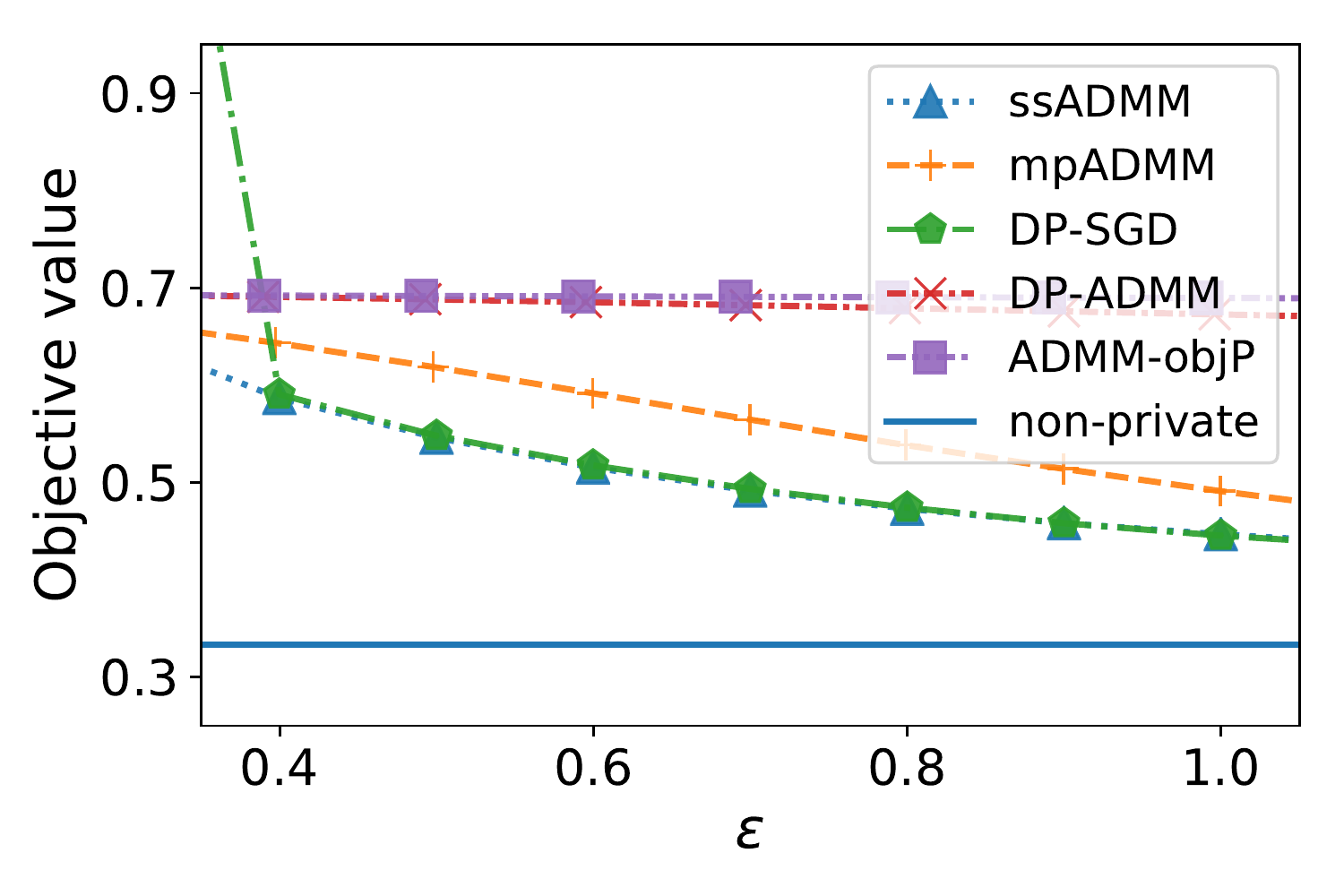}
        \caption{Objective value $\lambda=0.001$}
    \end{subfigure}
    \caption{Classification performance on simulated data}
    \label{figure_sim_perf}
\end{figure*}

\begin{figure*}[ht]
    \centering
    \begin{subfigure}[b]{0.245\textwidth}
        \includegraphics[width=\textwidth]{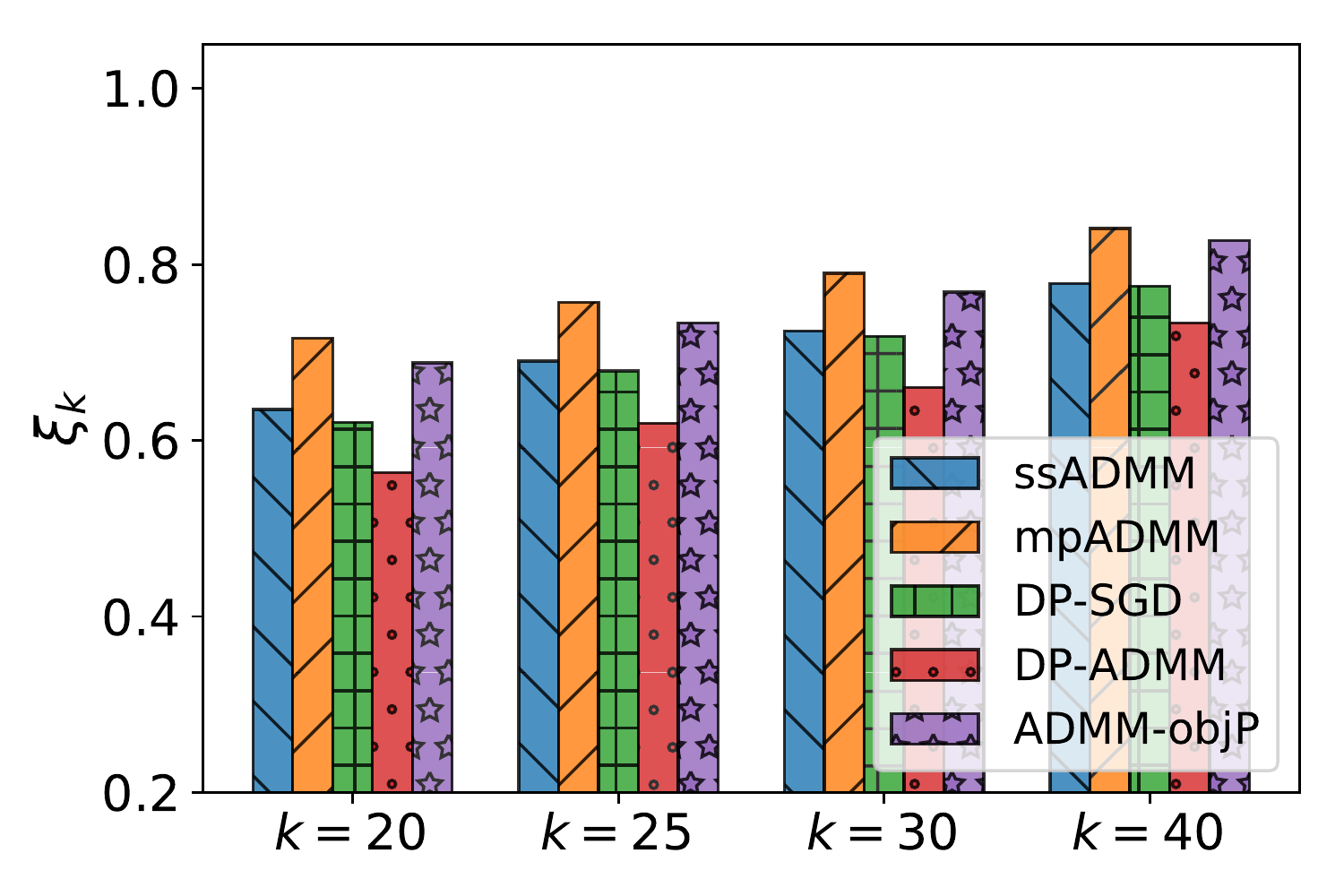}
        \includegraphics[width=\textwidth]{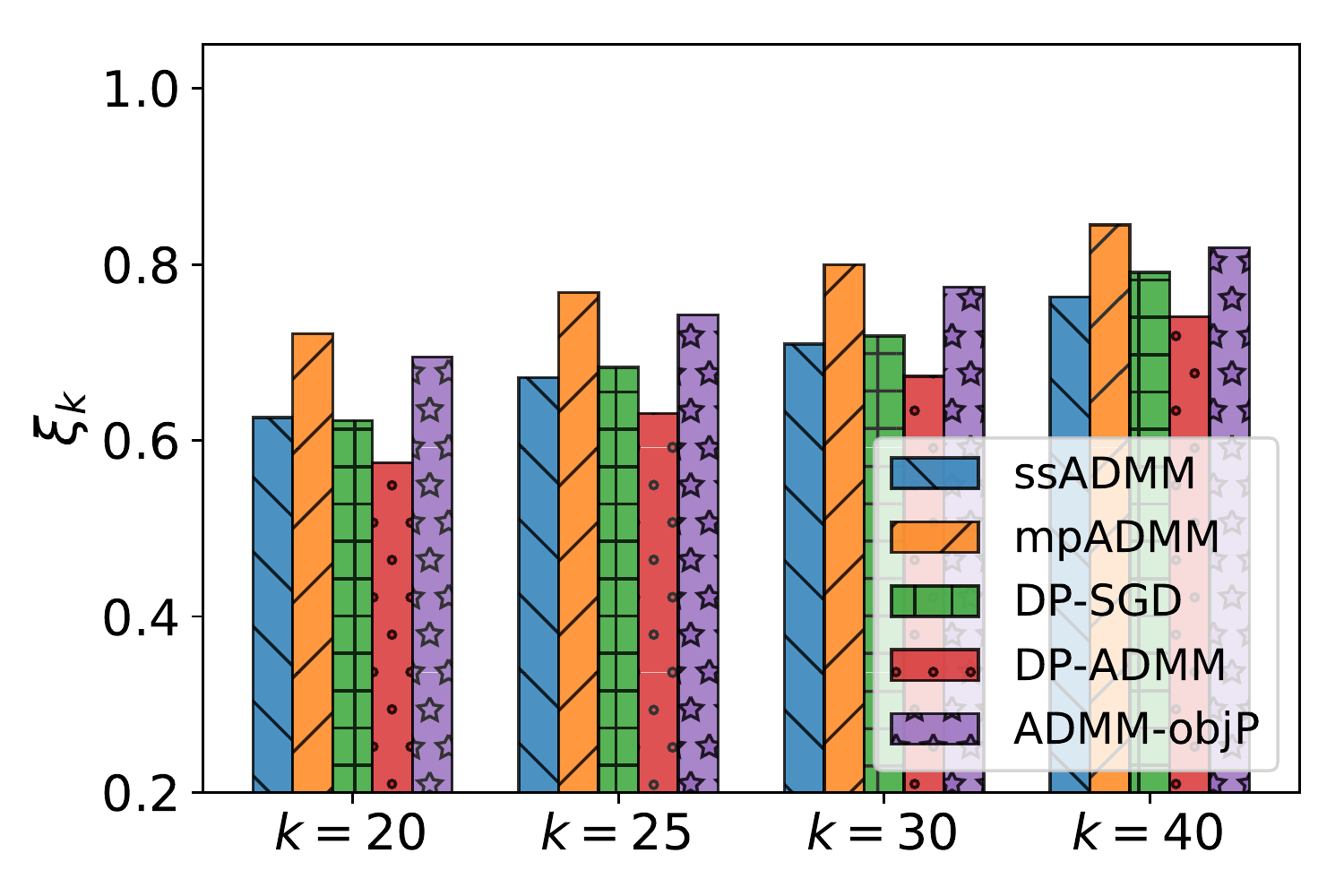}
        \caption{$\epsilon=0.4$}
    \end{subfigure}
    \begin{subfigure}[b]{0.245\textwidth}
        \includegraphics[width=\textwidth]{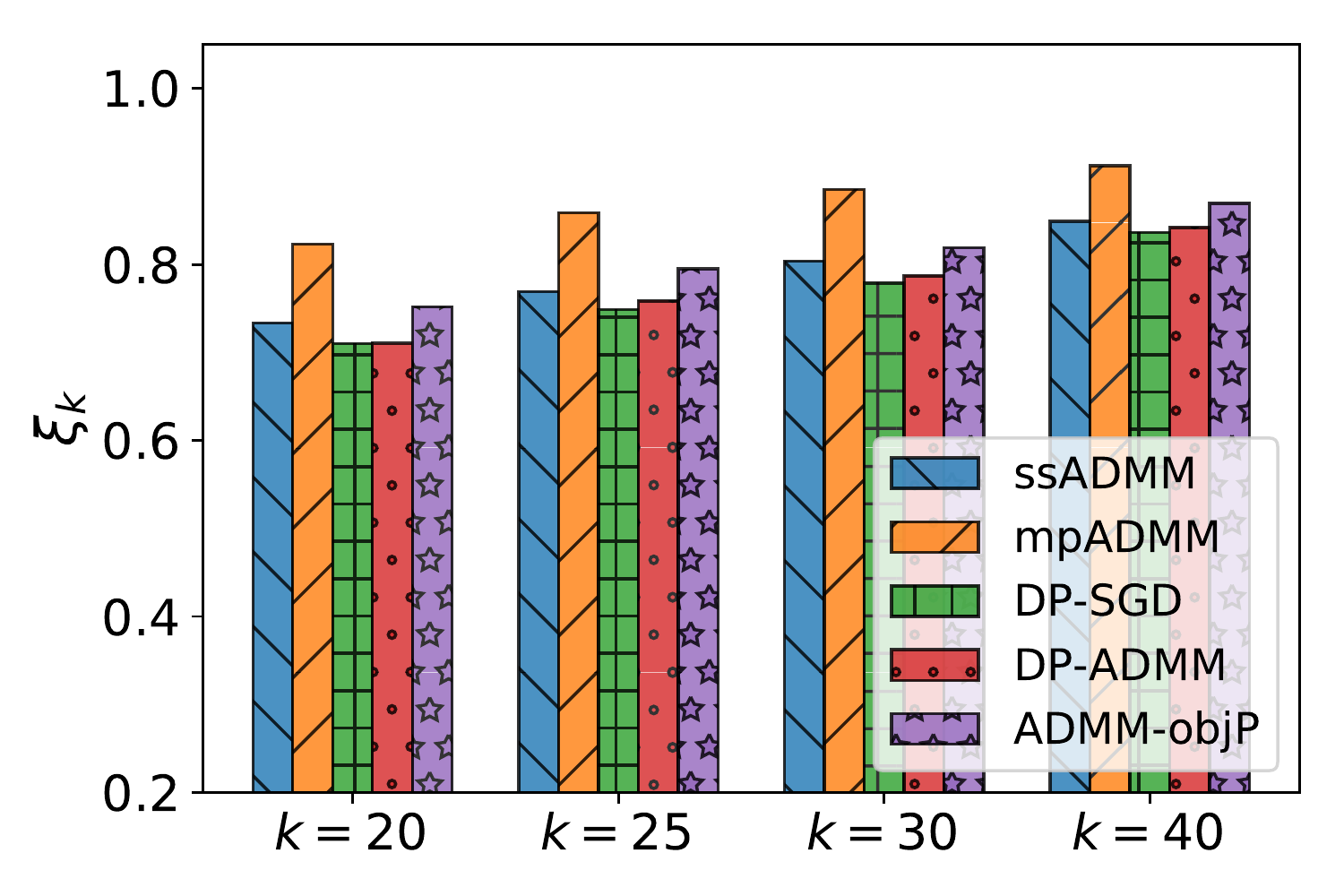}
        \includegraphics[width=\textwidth]{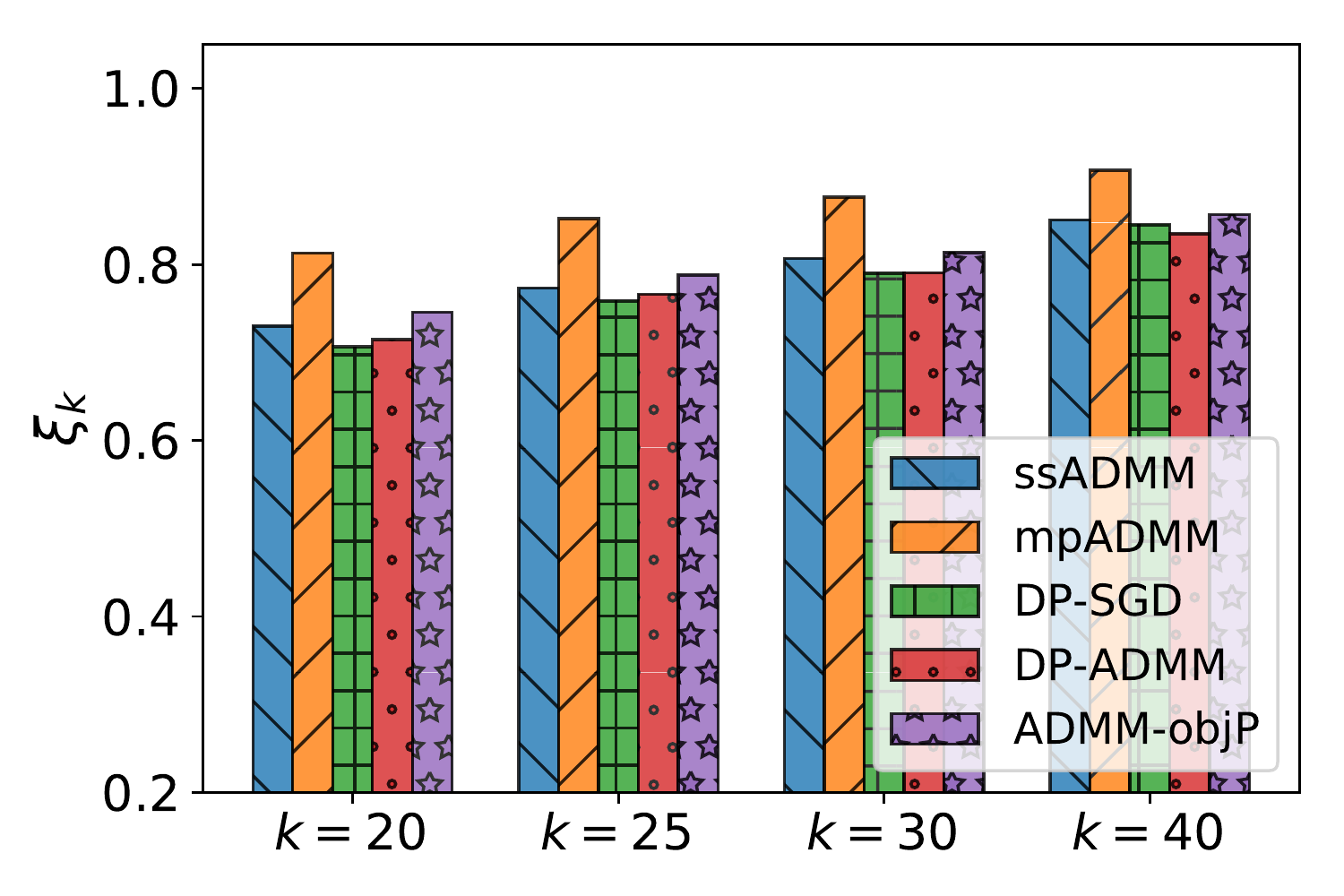}
        \caption{$\epsilon=0.6$}
    \end{subfigure}
    \begin{subfigure}[b]{0.245\textwidth}
        \includegraphics[width=\textwidth]{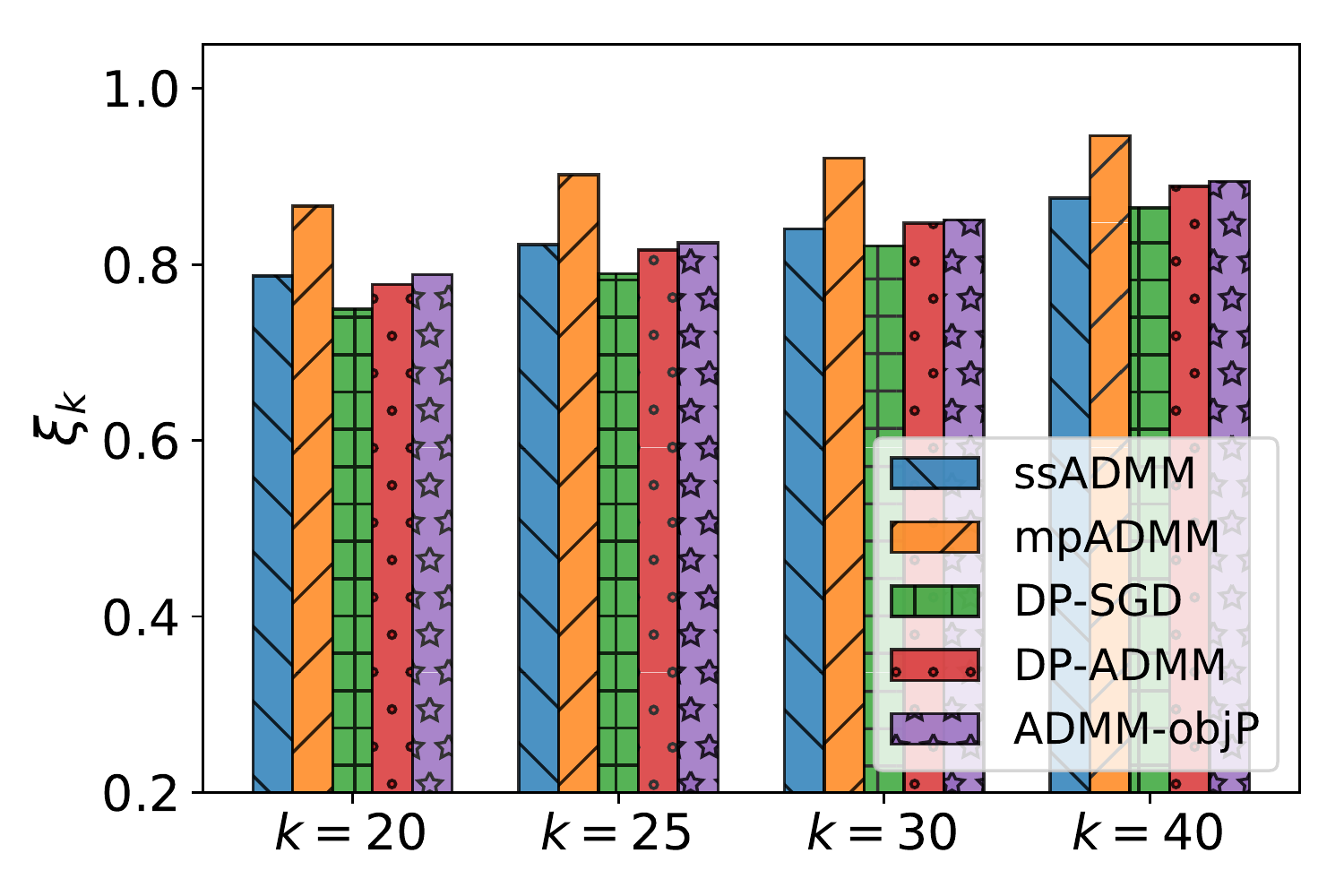}
        \includegraphics[width=\textwidth]{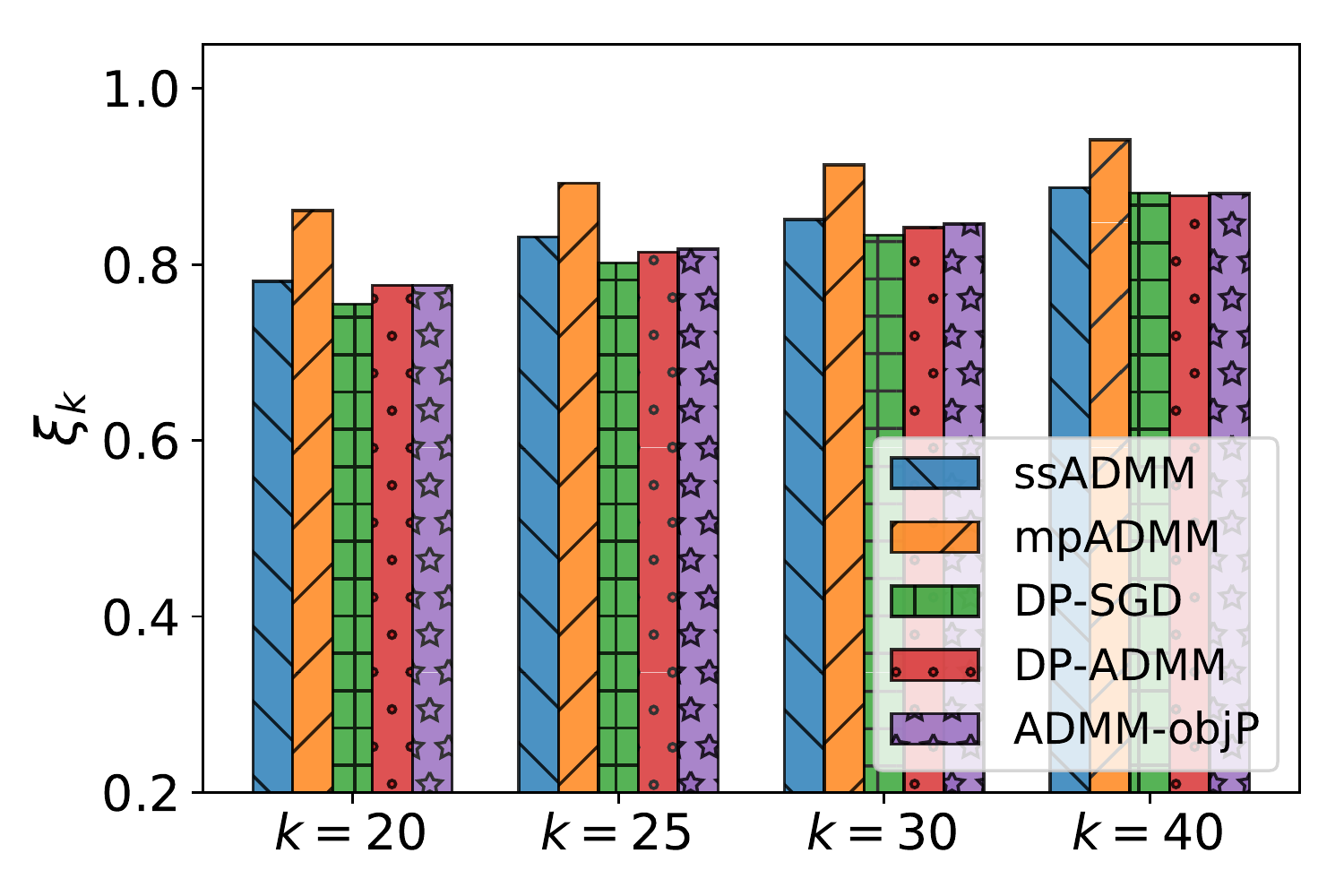}
        \caption{$\epsilon=0.8$}
    \end{subfigure}
    \begin{subfigure}[b]{0.245\textwidth}
        \includegraphics[width=\textwidth]{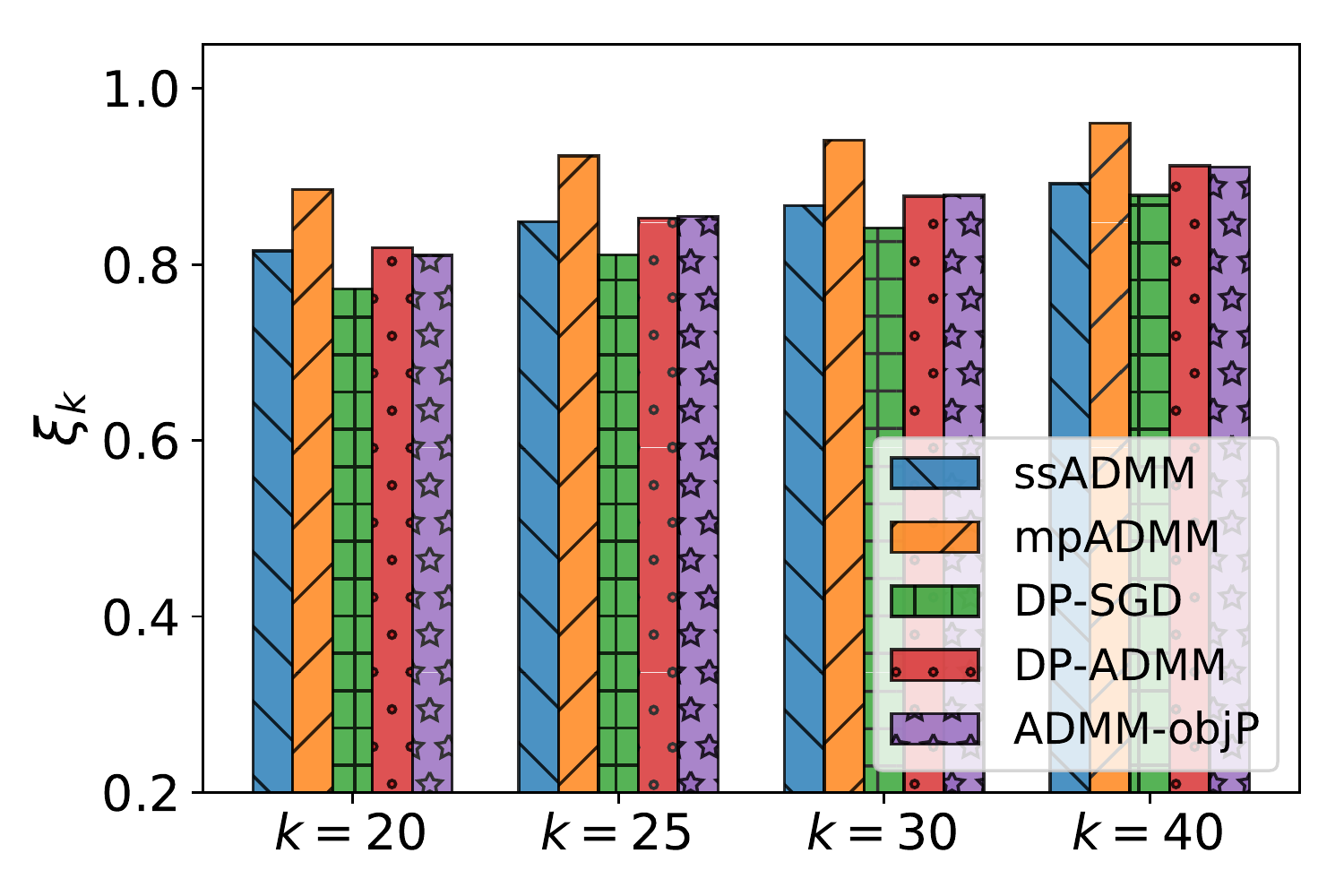}
        \includegraphics[width=\textwidth]{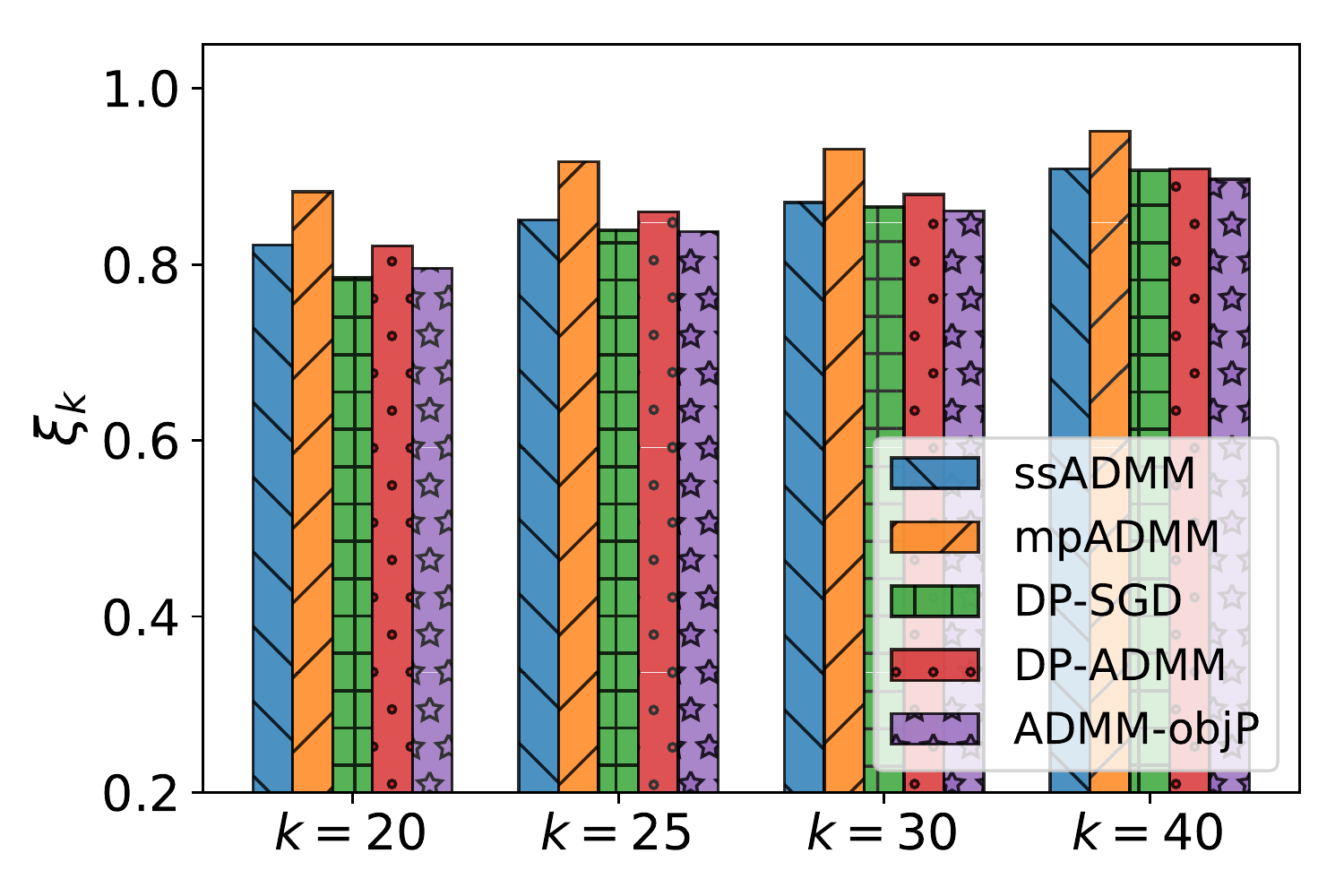}
        \caption{$\epsilon=1.0$}
    \end{subfigure}
    \caption{Attribute selection performance on simulated data (Top: $\lambda=0.0001$; Bottom: $\lambda=0.001$)}
    \label{figure_sim_as}
\end{figure*}

To measure the attribute selection performance, we test how many relevant attributes are selected by each algorithm for $L_1$ regularized logistic regression. Since the dataset is standardized, we can use the magnitude of the coefficient to rank the attributes, due to that noisy perturbation might cause the coefficients of irrelevant attributes slightly differ from zero. 

We define a criterion $\xi_{k}$ to measure the coverage of relevant attributes if top $k$ attributes suggested by the algorithm were selected. 
For example, since we know there are 20 relevant attributes in the simulated data, if we select $k=30$ attributes by magnitude of coefficient, 16 of them are the true relevant ones (i.e. among $x_1, ..., x_{20}$), then $\xi_{30}=16/20=0.8$. 
This make sense because in real case, the number of attributes we choose to select from an attribute ranker depends on the budget we can spend to collect data. We test all algorithms for $k=$ 20, 25, 30, and 40.

Figure \ref{figure_sim_perf} shows the classification performance of each algorithm on the simulated data. For non-private performance, we assume the true model is known. We can see that ssADMM, mpADMM, and DP-SGD have similar performance in classification. Figure \ref{figure_sim_as} shows the performance of attribute selection. Although classification accuracy are close, we can see that mpADMM can detect more relevant attributes, especially in the lower $\epsilon$ range. ADMM-objP, which was originally proposed for feature selection, can outperform ssADMM and DP-SGD for feature selection in low $\epsilon$ while its classification accuracy is behind ssADMM and DP-SGD. However, ADMM-objP usually require much more epochs in training compare to the other algorithms. Therefore, if we know the data is sparse and the major goal is focused on attribute selection, mpADMM is more preferable.

\section{Conclusions}
\label{sec:conclusion}

We present two privatizations of stochastic ADMM under
R\'enyi differential privacy. One algorithm combines gradient
perturbation technique with privacy amplification result to reduce the
total privacy loss throughout the execution.
The other algorithm uses the output perturbation (with numerical
computation of sensitivity) to privately release the solution at the
end of each training epoch.
These algorithms can be used to solve optimization problems with
complex structural regularization that induces sparsity.

\bibliographystyle{unsrt}  
\bibliography{references.bib}  

\appendix

\section{Proof of Theorem 2}

The proof is done by applying similar technique for Theorem 1 in \cite{ouyang2013stochastic}, plus considering the Gaussian noise term added. Define 
$$u:=\Colvec{x, z}, \overline{u}^k := \Colvec{\frac{1}{k}\sum_{i=1}^{k-1}x^i, \frac{1}{k}\sum_{i=1}^{k-1}z^i},\theta(u) := f(x) + h(z),$$
and define
$$w:=\Colvec{x, z, y}, \overline{w}^k := \Colvec{\frac{1}{k}\sum_{i=1}^{k-1}x^i, \frac{1}{k}\sum_{i=1}^{k-1}z^i, \frac{1}{k}\sum_{i=1}^{k-1}y^i}, F(w) := \Colvec{-y, y, x-z}$$
Denote $u^*:=\Colvec{x^*, z^*}$ as the optimal solution, and $\delta_{k+1} := \nabla f(x^{k}, B_k) - \nabla f(x^{k}, D)$, $d_{\mathcal X} := \sup_{x_a, x_b\in \mathcal X}\|x_a - x_b\|$, $d_{y^*} := \|y^0 - y^*\|$.

Therefore, consider the expectation of $\theta(\overline{u}^t) - \theta(u^*)$ after $t$ iterations,
\begin{equation}\label{eq_appendix1}
\begin{split}
    & \mathbb E\bigg[
    \theta(\overline{u}^t) - \theta(u^*) + (\overline{w}^t - w^*)^TF(\overline{w}^t)
    \bigg] \\
    = & \mathbb E
    \bigg[ 
    \theta(\overline{u}^t) - \theta(u^*) + (\overline{x}^t - x^*)^T(-\overline{y}^t) + (\overline{z}^t - z^*)^T(\overline{y}^t) \\
    & + (\overline{y}-y)^T(\overline{x}^t-\overline{z}^t) 
    \bigg] \\
    \leq & \mathbb E\bigg[
    \frac{1}{t}\sum_{k=0}^{t-1} \big[
    \frac{\eta^k}{2}\|\nabla f(x^k, B_k) + \gamma^k\|^2 + \frac{1}{2\eta^k}(\|x^k - x^*\|^2 - \|x^{k+1} - x^*\|^2) \\
    & + \langle \delta_{k+1}, x^* - x^k\rangle
    \big] + \frac{1}{t}\big( \frac{\rho}{2} \|x^*-z^0\|^2 + \frac{1}{2\rho} \|y-y^0\|^2 \big)
    \bigg] \\
    \leq & \mathbb E\bigg[
    \frac{1}{t}\sum_{k=0}^{t-1} \big[
    \frac{\eta^k(C^2 + p\sigma^2)}{2} + \langle \delta_{k+1}, x^* - x^k\rangle \big] \\
    & + \frac{1}{t} \big(
    \frac{d^2_{\mathcal X}}{2\eta^{t-1}} + \frac{\rho}{2}d^2_{y^*} + \frac{1}{2\rho}\|y-y^0\|^2
    \big)
    \bigg] \\
    = & \mathbb E\bigg[
    \frac{1}{t}\sum_{k=0}^{t-1} \big[
    \frac{\eta^k(C^2 + p\sigma^2)}{2} \big]
    + \frac{1}{t} \big(
    \frac{d^2_{\mathcal X}}{2\eta^{t-1}} + \frac{\rho}{2}d^2_{y^*} + \frac{1}{2\rho}\|y-y^0\|^2
    \big)
    \bigg]
\end{split}
\end{equation}

while the first inequality holds by applying an expected version of Lemma 2 in \cite{ouyang2013stochastic}, note that since noisy perturbation $\gamma\sim \mathcal N(0, \sigma^2\textbf{I}_p)$, $\mathbb E[\nabla f(x^k, B_k) + \gamma] = \nabla f(x^k, B_k)$, and $\mathbb E[\|\nabla f(x^k, B_k) + \gamma^k\|^2] \leq \mathbb E[\|\nabla f(x^k, B_k)\|^2] + \mathbb E[\|\gamma\|^2] + 2 \mathbb E[\|\nabla f(x^k, B_k)\|] \mathbb E[\gamma] \leq C^2 + p\sigma^2$. The last equality holds because we assume $x^k$ is independent of $B_k$ (which was used to calculate $x^{k+1}$) is independent of $x^k$, hence $\mathbb E_{B_k | B_{[0:k-1]}} \langle \delta_{k+1}, x^*-x^k\rangle = 0$.

The above holds for all dual variable $y$, hence it holds for $y$ in a ball $\mathcal B_0 = \{y: \|y\|_2\leq \beta\}$. According to (33) in \cite{ouyang2013stochastic}, 
\begin{equation}
    \max_{y\in \mathcal B_0} \{ \theta(\overline{u}^t) - \theta(u^*) + (\overline{w}^t - w^*)^T F(\overline{w}^t) \} = \theta(\overline{u}^t) - \theta(u^*) + \beta\|\overline{x}_t - \overline{z}_t \|
\end{equation}
Therefore, continue on (\ref{eq_appendix1}), we can have
\begin{equation}
\begin{split}
    & \mathbb E\big[
    \theta(\overline{u}^t) - \theta(u^*) + \beta\|\overline{x}_t - \overline{z}_t \|
    \big] \\
    \leq & \mathbb E\bigg[
    \frac{1}{t}\sum_{k=0}^{t-1} \big[
    \frac{\eta^k(C^2 + p\sigma^2)}{2} \big]
    + \frac{1}{t} \big(
    \frac{d^2_{\mathcal X}}{2\eta^{t-1}} + \frac{\rho}{2}d^2_{y^*} + \frac{1}{2\rho}\|y-y^0\|^2
    \big)
    \bigg] \\
    \leq & \mathbb E\bigg[
    \frac{1}{t}\sum_{k=0}^{t-1} \big[
    \frac{\eta^k(C^2 + p\sigma^2)}{2} \big]
    + \frac{1}{t} \big(
    \frac{d^2_{\mathcal X}}{2\eta^{t-1}} + \frac{\rho}{2}d^2_{y^*}
    \big)
    \bigg] \\
    & + \mathbb E\bigg[ \max_{y\in \mathcal B_0}\{\frac{1}{2\rho t}\|y-y_0\|^2 \bigg] \\
    \leq & \frac{1}{t}\bigg( \frac{C^2 + p\sigma^2}{2}\sum_{k=1}^t\eta^k + \frac{d^2_{\mathcal X}}{2\eta^{t-1}} \bigg) + \frac{\rho d_{y^*}^2}{2t} + \frac{\beta^2}{2\rho t}
\end{split}
\end{equation}

So if we choose $\eta^k = \frac{d_{\mathcal X}}{\sqrt{2(C^2 + p\sigma^2)k}} = O(1/\sqrt{k})$, $\mathbb E\big[
\theta(\overline{u}^t) - \theta(u^*) + \beta\|\overline{x}_t - \overline{z}_t \|
\big] \leq \frac{d_{\mathcal X}\sqrt{2(C^2+p\sigma^2)}}{\sqrt{t}} + \frac{\rho d_{y^*}^2}{2t} + \frac{\beta^2}{2\rho t} = O(1/\sqrt{t})$.

\end{document}